\documentclass[letterpaper]{article} 
\usepackage{aaai2026}  
\usepackage{times}  
\usepackage{helvet}  
\usepackage{courier}  
\usepackage[hyphens]{url}  
\usepackage{graphicx} 
\usepackage{natbib}  
\usepackage{caption} 
\usepackage{algorithm}
\usepackage{algorithmic}
\usepackage{multirow}
\usepackage{amsthm}
\usepackage{mathrsfs}
\usepackage{amsmath}
\usepackage{subfig}
\usepackage{amssymb}
\usepackage{booktabs}
\usepackage{bm}
\newtheorem{theorem}{Theorem}
\newtheorem{proposition}{Proposition}
\newtheorem{definition}{Definition}
\allowdisplaybreaks[4]
\usepackage[noabbrev]{cleveref}
\usepackage{color}
\usepackage{xcolor}
\usepackage{pdfpages}
\usepackage{appendix}
\usepackage{newfloat}
\usepackage{listings}
\DeclareCaptionStyle{ruled}{labelfont=normalfont,labelsep=colon,strut=off} 
\floatstyle{ruled}
\newfloat{listing}{tb}{lst}{}
\floatname{listing}{Listing}
\title{DeepFaith: A Domain-Free and Model-Agnostic Unified Framework\\for Highly Faithful Explanations}
\author{
    Yuhan Guo,
    Lizhong Ding,
    Shihan Jia\equalcontrib,
    Yanyu Ren\equalcontrib,
    Pengqi Li,
    Jiarun Fu,\\
    Changsheng Li,
    Ye yuan,
    Guoren Wang
}
\affiliations{
    Beijing Institute of Technology\\

    No 5Zhongguancun South Street, Haidian District\\
    Beijing, China\\
    3120255836@bit.edu.cn
}

\usepackage{bibentry}

\begin{document}

\maketitle

\begin{abstract}
Explainable AI (XAI) builds trust in complex systems through model attribution methods that reveal the decision rationale. However, due to the absence of a unified optimal explanation, existing XAI methods lack a ground truth for objective evaluation and optimization. To address this issue, we propose \underline{Deep} architecture-based \underline{Faith}ful explainer (\textbf{DeepFaith}), a domain-free and model-agnostic unified explanation framework under the lens of faithfulness. By establishing a unified formulation for multiple widely used and well-validated faithfulness metrics, we derive an optimal explanation objective whose solution simultaneously achieves optimal faithfulness across these metrics, thereby providing a ground truth from a theoretical perspective. We design an explainer learning framework that leverages multiple existing explanation methods, applies deduplicating and filtering to construct high-quality supervised explanation signals, and optimizes both \textit{pattern consistency loss} and \textit{local correlation loss} to train a faithful explainer. Once trained, \textbf{DeepFaith} can generate highly faithful explanations through a single forward pass without accessing the model being explained. On 12 diverse explanation tasks spanning 6 models and 6 datasets, \textbf{DeepFaith} achieves the highest overall faithfulness across 10 metrics compared to all baseline methods, highlighting its effectiveness and cross-domain generalizability.
\end{abstract}

\section{Introduction}

As deep learning models are increasingly applied in high-risk fields such as healthcare \cite{health2, health3}, finance \cite{finance1, finance3}, and criminal justice \cite{jus1, jus2}, eXplainable Artificial Intelligence (XAI) has become a core requirement to ensure their trustworthiness, fairness, and safety \cite{xai1, xai2, xai3}. However, the explainability of machine learning faces the fundamental challenge of the absence of a \textit{Ground Truth} \cite{Li2023}, leading to different explanation methods relying on manually set prior assumptions \cite{Selvaraju1610, Lundberg1705, chen2024less}, resulting in a lack of a unified optimization objective.

\begin{figure}[ht]
\centering
\label{layerResult:subfig2}\includegraphics[width=0.96\columnwidth]{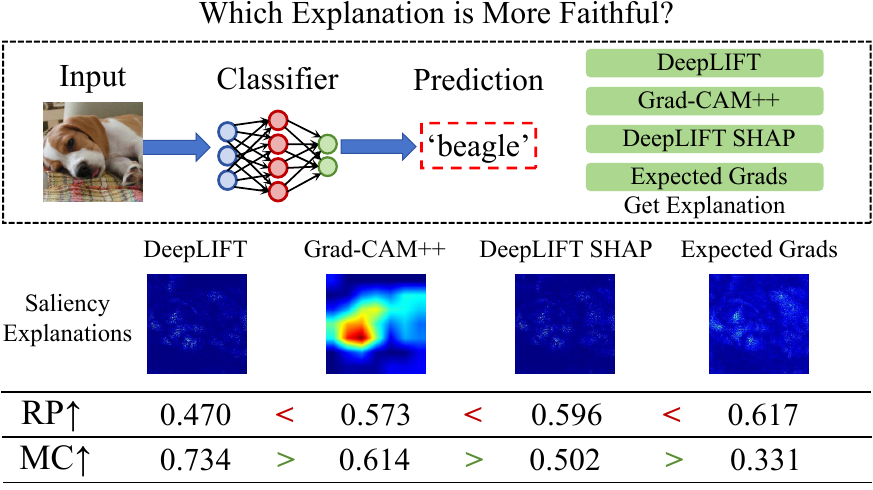}
\caption{Explanations from four methods for an image classifier prediction, along with their faithfulness scores assessed by Region Perturbation (RP) \cite{Samek1509} and Monotonicity Correlation (MC) \cite{Nguyen2007}, with higher values indicating greater faithfulness.}
\label{faithConf}
\end{figure}

Faithfulness evaluation \cite{Bhatt3016, Dasgupta2022} quantifies the alignment between explanations and model decisions via perturbation experiments, offering a practical alternative to ground truth \cite{Li2023}. However, as shown in Figure \ref{faithConf}, different metrics often produce conflicting results \cite{Klein2024}, providing little unified guidance for explanation optimization and leaving the issue unresolved.

\begin{table*}[htb]
\centering
{\fontsize{9pt}{14pt}\selectfont
\setlength{\tabcolsep}{2pt}
    \begin{tabular}{lcccr}
    \toprule
    Metric & Input & Formula & Output \\
    \midrule
     Faithfulness Correlation (FC)  & $s;x,f$     & $\tau\left [ \left (\sum_{i\in\mathcal{I}}s_i\right )_{\mathcal{I}\subseteq [n]},\left ( \Delta \left [ f(x),f(x\setminus \mathcal{I}) \right ]\right )_{\mathcal{I}\subseteq [n]} \right ]$     & $[-1,1]$ \\
    Faithfulness Estimate (FE)     & $S_f;\{x^{(i)},\mathcal{I}_i\}_{i=1}^N,f$     & $\tau \left [  (\sum_{j\in\mathcal{I}_i}S_f(x^{(i)})_j )_{i=1}^N,( \Delta \left [ f(x^{(i)}),f(x^{(i)}\setminus \mathcal{I}_i) \right ] )_{i=1}^N \right ]$     & $[-1,1]$  \\
    Infidelity (INF)     & $s;x,\{\mathcal{I}_i\sim \mathcal{P}([n])\}_{i=1}^N,f$     & $ \tau [  ( \sum_{j\in\mathcal{I}_i}s_j  )_{i=1}^N, \left ( \Delta \left [ f(x),f(x\setminus \mathcal{I}_i) \right ] )_{i=1}^N \right ]$     & $[-1,1]$  \\
    Monotonicity Correlation (MC)     & $s;x,\{\mathcal{I}_i\}_{i=1}^N,f$     & $ \tau [  ( \sum_{j\in\mathcal{I}_i}s_j  )_{i=1}^N, \left ( \Delta \left [ f(x),f(x\setminus \mathcal{I}_i) \right ] )_{i=1}^N \right ]$     & $[-1,1]$  \\
    \midrule
    Deletion Score* (DEL)     & $\pi;x,f$     & $\frac{1}{n}\!\int_{i=0^{+}}^{n} \Delta^- \left [ f(x),f(x\setminus { \bigcup_{j=1}^{\left \lceil i \right \rceil }\pi(j)})\right ] \mathrm{d}i$     & $[0,1]$ \\
    Insertion Score* (INS)     & $\pi;x,f$     & $\frac{1}{n}\int_{i=0^{+}}^{n} \Delta^- \left [ f(x),f(x^{\circ}\cup { \bigcup_{j=1}^{\left \lceil i \right \rceil }\pi(j)})\right ] \mathrm{d}i$     & $[0,1]$  \\
    Negative Perturbation* (NEG)     & $\pi;x,f$     & $\frac{1}{t}\int_{i=0^{+}}^{t} \Delta^- \left [ f(x),f(x\setminus { \bigcup_{j=1}^{\left \lceil i \right \rceil }\overset{\hookleftarrow }{\pi} (j)})\right ] \mathrm{d}i$     & $[0,1]$   \\
    Positive Perturbation* (POS)     & $\pi;x,f$     & $\frac{1}{t}\int_{i=0^{+}}^{t} \Delta^- \left [ f(x),f(x\setminus { \bigcup_{j=1}^{\left \lceil i \right \rceil }\pi(j)})\right ] \mathrm{d}i$     & $[0,1]$  \\
    Region Perturbation (RP)      & $\Pi_f;\{x^{(i)}\}_{i=1}^N,f$     & $\frac{1}{N}\sum_{i=1}^N\! \left (\! \frac{1}{n+1}\!\sum_{j=0}^n \Delta\!\! \left [ f(x^{(i)}),f(x^{(i)}\!\setminus\! { \bigcup_{k=1}^{j}\Pi_f(x^{(i)})(k)} ) \right ] \right )$     & $[0,1]$  \\
    Iterative Removal of Features (IROF)     & $\Pi_f;\{x^{(i)}\}_{i=1}^N,f$     & $\frac{1}{Nn}\!\sum_{i=1}^N \!\!\int_{j=0^+}^n \!1\!-\!\Delta^-\!\!\left [f(x^{(i)}),f(x^{(i)}\!\setminus\! { \bigcup_{k=1}^{\left \lceil j \right \rceil }\Pi_f(x^{(i)})(k)} ) \right ]\mathrm{d}j$     & $[0,1]$  \\
    \bottomrule
    \end{tabular}%
}
\caption{We formalize for the first time four widely used and well-validated faithfulness metrics (*) and re-formalize six ones under our unified framework, including FC \cite{Bhatt3016}, FE \cite{AlvarezMelis1806}, INF \cite{Yeh1901}, and MC for saliency explanations, as well as DEL and INS \cite{Petsiuk1806}, NEG and POS \cite{Barkan2310}, RP, and IROF \cite{rieger2020} for permutation explanations. Here, $\tau$ denotes a correlation metric, $\Delta$ a perturbation effect, $\Delta^-$ a preservation effect, $\mathcal{P}$ the uniform distribution over the power set, $x^{\circ}$ the baseline input, $\overset{\hookleftarrow }{\pi}$ the reversed permutation explanation, and $t$ the least number of perturbations required to change the model prediction significantly.}
\label{faithTable} 
\end{table*}

We observe that various widely used faithfulness metrics can be unified under a specific theoretical framework, which enables deriving an objective for the optimal faithfulness, thus offering a surrogate for ground truth. Furthermore, despite methodological differences, existing explanation techniques consistently capture the functional relationship between input features and model predictions. This shared pattern suggests the feasibility of learning a generalizable mapping from inputs to high-quality explanations.

Building on these insights, we propose \underline{Deep} architecture-based \underline{Faith}ful explainer (\textbf{DeepFaith}), a domain-free and model-agnostic unified framework for generating highly faithful explanations. We rigorously distinguish faithfulness metrics evaluating \textit{saliency} and \textit{permutation} explanations, formalize four empirical ones for the first time, and re-formalize six metrics within our theoretical framework. We propose and prove that a saliency explanation mapping achieves optimal faithfulness across all metrics. Moreover, we design an explainer learning framework that leverages multiple baseline explanation methods to generate explanations and constructs high-quality supervised explanation signals through deduplicating and filtering. Integrating the optimal faithfulness objective and the patterns of supervised explanation signals, we train a deep neural network explainer by optimizing two corresponding loss functions. Once trained, \textbf{DeepFaith} generates highly faithful explanations for inputs via a single forward pass, without accessing the model being explained.

We evaluate \textbf{DeepFaith} on 12 explanation tasks spanning image, text, and tabular modalities, as well as diverse models being explained. Comparative experiments demonstrate that \textbf{DeepFaith} consistently achieves higher faithfulness than baseline methods while providing clear and intuitive visualizations. Furthermore, we provide a runtime efficiency comparison of \textbf{DeepFaith} for explanation inference, along with ablation studies on the two loss components.

\section{Unified Formulation of Faithfulness Metrics}
In this section, we propose a domain-free and model-agnostic framework that unifies multiple widely used and well-validated faithfulness evaluation metrics. Let $f:\mathcal{X}\to \mathcal{Y}$ denote the model to be explained, where the input space $\mathcal{X}\subseteq \mathbb{R}^{n\times d}$ consists of instance $x=(x_1,x_2,\dots,x_n)$ with each element $x_i \in \mathbb{R}^d$. In our experiments: for vision, $x$ is an image of $n$ patches, each $x_i$ representing the $d$-dimensional pixels in a patch; for NLP, $x$ is a sequence of $n$ tokens with $x_i$ as the $d$-dimensional embedding of the $i$-th token; for tabular data, $x$ is a row with $n$ scalar features ($d=1$). The model output $f(x)$ aims to approximate $y \in \mathcal{Y} \subseteq \mathbb{R}$, e.g., the predicted probability for the target class in classification. We use $[n]$ to denote the set $\{1,2,\dots,n\}$, and use $(i)_{i=1}^n$ to denote the vector $(1,2,\dots,n)$.

We begin with the observation that current metrics follow two distinct views: one evaluates the accuracy of attribution values from a saliency perspective \cite{Bhatt3016, AlvarezMelis1806}, while the other assesses the relative importance of input elements from a permutation perspective \cite{Samek1509, rieger2020}. Thus, it is essential to distinguish between explanations under these two perspectives.

\begin{definition}[Saliency Explanation]
A saliency explanation method is defined as a mapping $S_f:\mathcal{X} \to [0,1]^n$ that, given an input $x$ and model $f$, outputs a saliency vector $s = (s_1, s_2, \dots, s_n) \in [0,1]^n$, where each $s_i$ quantifies the contribution of $x_i$ (e.g., a patch, token, or scalar feature) to the prediction $\hat{y} = f(x)$.
\end{definition}

\begin{definition}[Permutation Explanation]
A permutation explanation method is defined as a mapping $\Pi_f:\mathcal{X}\to \mathfrak{S}_n$, where $\mathfrak{S}_n=\{(\pi(i))_{i=1}^n|\{\pi(1),\pi(2),\cdots,\pi(n)\}=[n]\}$ denotes all permutations of $[n]$.\footnote{For clarity, we use $\pi(i)$ to denote the $i$-th element in vector $\pi$.} Given $x$ and model $f$, $\Pi_f$ outputs $\pi \in \mathfrak{S}_n$, indicating that $x_{\pi(i)}$ contributes no less to the model’s prediction than $x_{\pi(i+1)}$.
\end{definition}

Two types of explanations can be interconverted via simple functions: $\mathfrak{P}(s) = {\rm argsort}_\downarrow \{s_1, s_2, \ldots, s_n\}$ represents the descending-order index of $s$, mapping a saliency explanation to a permutation explanation; while $\Sigma(\pi)$ converts a permutation explanation into a saliency explanation, where $\Sigma(\pi)_{\pi(i)} = (n - \pi(i) + 1)/n$. Since a saliency explanation assigns a specific importance score to each $x_i$, while a permutation explanation does not, $\mathfrak{P}(s)$ cannot be recovered back to $s$ through $\Sigma$.

Our unified framework is built upon a notation system derived from a deep understanding of faithfulness evaluation. Let $x \setminus \mathcal{I}$ ($\mathcal{I} \subseteq [n]$) denote input $x$ with sub-elements $\{x_i |i \in \mathcal{I}\}$ removed (via noise substitution \cite{Rong2202}, baseline replacement \cite{Bhatt3016,Bach2015}, or linear interpolation \cite{rieger2020}). We define perturbation effect $\Delta:\mathcal{Y}\times\mathcal{Y}\to [0,1]$ (e.g., $|y^{(1)}-y^{(2)}|$ or $\frac{1}{2}(y^{(1)}-y^{(2)})^2$ \cite{Yeh1901}) and preservation effect $\Delta^-:\mathcal{Y}\times\mathcal{Y}\to [0,1]$, negatively correlated with $\Delta$, measure the extent to which the original prediction is preserved (e.g., $\left | y^{(1)}/y^{(2)} \right |$ \cite{rieger2020} or target class confidence). We also define $\tau:\mathbb{R}^m\times\mathbb{R}^m\to [-1,1]$ to measure correlations between $m$-dimensional vectors, such as Pearson or Spearman coefficients \cite{AlvarezMelis1806, Nguyen2007}.

We re-formalize four saliency perspective faithfulness metrics under our unified framework, as shown in Table \ref{faithTable}. Specifically, FC enumerates all subsets of $[n]$ as perturbation index sets $\mathcal{I}$; FE evaluates $N$ samples, each with a specific $\mathcal{I}$; MC defines a fixed perturbation sequence $\{\mathcal{I}_i\}_{i=1}^N$ on one sample; and INF samples $N$ index sets from a distribution $\mathcal{P}$, which we instantiate as $\mathcal{P}([n]) = \mathrm{Uniform}(2^{[n]})$, a discretized version of the original INF.

For permutation perspective metrics, we reformulate two existing ones and, for the first time, formalize four empirical metrics. In Table \ref{faithTable}, RP perturbs features in descending order of importance and averages the prediction drop; IROF uses the same order and computes the mean area over the curve (AOC) of preservation effects across $N$ samples; DEL and INS respectively remove features from the original input $x$ or insert features into a baseline input $x^{\circ}$ (e.g., blurred input, noise, or zero vector), using the area under the curve (AUC) of preservation effects as faithfulness scores; NEG and POS remove features in ascending or descending order until $t$-th removal leading to prediction changes significantly (e.g., class flips), with AUC used to quantify the effect.

\section{Theoretical Analyses of Optimal Faithfulness}
Building on our unified framework of faithfulness evaluation, we propose and theoretically establish the existence of an \textit{optimal explanation mapping}.

By uncovering that the core idea behind FC, FE, INF, and MC is to evaluate the correlation between the local sum of saliency explanations over perturbed indices and the corresponding perturbation effects, we propose a saliency explanation mapping with optimal faithfulness as follows.

\begin{proposition}
\label{saProp}
\textit{Given a model $f$ being explained and its input space $\mathcal{X}$, for a fixed correlation measure \( \tau \) and perturbation effect \( \Delta \), suppose there exists a saliency explanation mapping \( S_f^* \) such that $\forall x\in\mathcal{X}$ and $\forall\{\mathcal{I}_i\subseteq[n]\}_{i=1}^N$,}
\begin{equation}
\label{optSaliency}
S_f^*\!=\!\underset{S_f}{\mathrm{argmax}}\ \tau\!\! \left [  ({ {\textstyle \sum}_{j\in\mathcal{I}_i}}S_f^*(x) )_{i=1}^N,\!\left ( \Delta[f(x),f(x\!\setminus\!\mathcal{I}_i)] \right )_{i=1}^N \right ],
\end{equation}
then the saliency explanations generated by \( S_f^* \) always achieve optimal faithfulness under the $\mathrm{FC}$, $\mathrm{FE}$, $\mathrm{INF}$, and $\mathrm{MC}$ evaluation metrics.
\end{proposition}

Although RP, IROF, DEL, INS, NEG, and POS evaluate permutation explanations in ways that differ substantially from FC, FE, INF, and MC, we theoretically show that they share an underlying consistency, and prove that $S_f^*$ in Proposition \ref{saProp} can induce an optimal permutation explanation mapping on all six permutation-based faithfulness metrics.

\begin{figure*}[htb]
\centering
\includegraphics[width=1\textwidth]{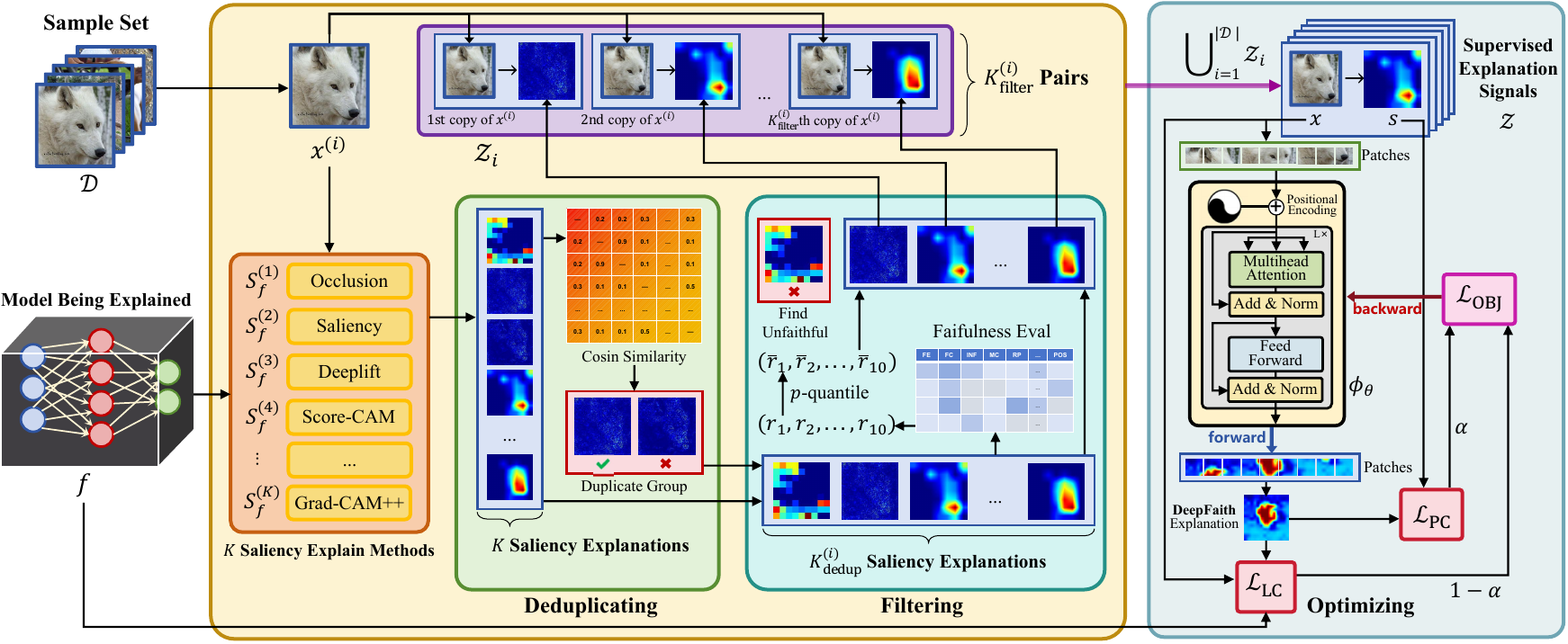}
\caption{\textbf{DeepFaith} learning framework. We meticulously design a high-quality supervised explanation signal generation workflow that leverages $K$ existing explanation methods with deduplicating and filtering. We further introduce a training pipeline for a deep neural explainer (an $L$-layer Transformer encoder in the figure) that optimizes $\mathcal{L}_{\mathrm{LC}}$ (Eq. \ref{faithObj}) theoretically grounded by Theorem \ref{puTheo} and $\mathcal{L}_{\mathrm{PC}}$ (Eq. \ref{simObj}) empirically guided by the supervised signals. Image modality is shown as an example.}
\label{deepfaith}
\end{figure*}

\begin{theorem}
\label{puTheo}
\textit{Under the conditions of Proposition \ref{saProp}, given a fixed preservation effect $\Delta^-$ that is negatively correlated with $\Delta$, let $\Pi_f^*(\cdot)=\mathfrak{P}[S_f^*(\cdot)]$ denote the permutation explanation mapping induced by $S_f^*$, then for any sample $x$, $\Pi_f^*(x)$ always achieve optimal faithfulness under the $\mathrm{DEL}$, $\mathrm{INS}$, $\mathrm{NEG}$, $\mathrm{POS}$, $\mathrm{RP}$ and $\mathrm{IROF}$ evaluation metrics.}
\end{theorem}

\begin{proof}
$\forall \Pi_f$, given an input sample $x\in\mathcal{X}$, let $\pi = \Pi_f(x)$ and $\pi^*=\Pi_f^*(x)$ denote the permutation explanations generated by different mappings, and $s^*=S_f^*(x)$. In addition, we denote $\Delta_{f,x}(\mathcal{I})=\Delta[f(x),f(x\setminus\mathcal{I})]$ for simplicity.

Given any $\mathcal{I}_a,\mathcal{I}_b$ satisfying $\sum_{j\in \mathcal{I}_a}s^*_j\ge \sum_{j\in \mathcal{I}_b}s^*_j$, suppose that $\Delta_{f,x}(\mathcal{I}_a) < \Delta_{f,x}(\mathcal{I}_b)$. Then there must exist $s$ satisfying $\sum_{j\in\mathcal{I}_a}s_j<\sum_{j\in\mathcal{I}_b}s_j$ such that
\begin{equation*}
\tau\!\! \left [ \!\!
\binom{\sum_{j\in \mathcal{I}_a}s_j}{\sum_{j\in \mathcal{I}_b}s_j}\!,\!\!
\binom{\!\Delta_{f,x}(\mathcal{I}_a)}{\!\Delta_{f,x}(\mathcal{I}_b)}\!\!
\right ]\!\!\!>\!
\tau \!\!\left [ \!\!
\binom{\sum_{j\in \mathcal{I}_a}s^*_j}{\sum_{j\in \mathcal{I}_b}s^*_j}\!,\!
\binom{\!\Delta_{f,x}(\mathcal{I}_a)}{\!\Delta_{f,x}(\mathcal{I}_b)}\!\!
\right ]\!,
\end{equation*} which contradicts the definition of $S_f^*$ given in Eq. (\ref{optSaliency}). Therefore, we can conclude that
\begin{equation*}
\forall\ \mathcal{I}_a,\mathcal{I}_b,\sum_{j\in \mathcal{I}_a}s^*_j\ge \sum_{j\in \mathcal{I}_b}s^*_j\Rightarrow \Delta_{f,x}(\mathcal{I}_a)\ge\Delta_{f,x}(\mathcal{I}_b).
\end{equation*}
Considering index sets $\bigcup_{j=1}^i\pi^*(j)$ and $\bigcup_{j=1}^i\pi(j)$, since the permutation explanation $\pi^*=\mathfrak{P}(s^*)$ implies that $\forall i\le n$, $\sum_{j=1}^i s^*_{\pi^*(j)}\ge \sum_{j=1}^i s^*_{\pi(j)}$, thus we have
\begin{equation*}
\Delta_{f,x}\left ({\textstyle \bigcup_{j=1}^i\pi^*(j)}\right )\ge \Delta_{f,x}\left ({\textstyle \bigcup_{j=1}^i\pi(j)}\right ).
\end{equation*} By aggregating this result over samples $\{x^{(i)}\}_{i=1}^N$, we can get $\mathrm{RP}(\Pi_f^*;\{x^{(i)}\}_{i=1}^N,f)\ge \mathrm{RP}(\Pi_f;\{x^{(i)}\}_{i=1}^N,f)$. Since $\Delta^-$ is negatively correlated with $\Delta$, i.e., 
\begin{equation*}
\forall\ \mathcal{I}_a,\mathcal{I}_b,\sum_{j\in \mathcal{I}_a}s^*_j\ge \sum_{j\in \mathcal{I}_b}s^*_j\Rightarrow \Delta^-_{f,x}(\mathcal{I}_a)\le\Delta^-_{f,x}(\mathcal{I}_b),
\end{equation*} it is obvious that $\mathrm{DEL}(\pi^*;x,f)\le \mathrm{DEL}(\pi;x,f)$ and $\mathrm{POS}(\pi^*;x,f)\le \mathrm{POS}(\pi;x,f)$; by the same way, one can derive $\mathrm{NEG}(\pi^*;x,f)\ge \mathrm{NEG}(\pi;x,f)$ and $\mathrm{IROF}(\Pi_f^*;\{x^{(i)}\}_{i=1}^N,f)\ge \mathrm{IROF}(\Pi_f;\{x^{(i)}\}_{i=1}^N,f)$.

Given a baseline input $x^{\circ}$ representing the uninformative state, we have $x^{\circ}\cup {\bigcup_{j=1}^{i}\pi(j)}=x\setminus\bigcup_{j=1}^i\overset{\hookleftarrow }{\pi}$, thus $\mathrm{INS}(\pi^*;x,f)\ge \mathrm{INS}(\pi;x,f)$.
\end{proof}


\begin{figure}[ht]
\centering
\subfloat[Function Space]
  {
      \label{funcSpace}\includegraphics[width=0.5\columnwidth]{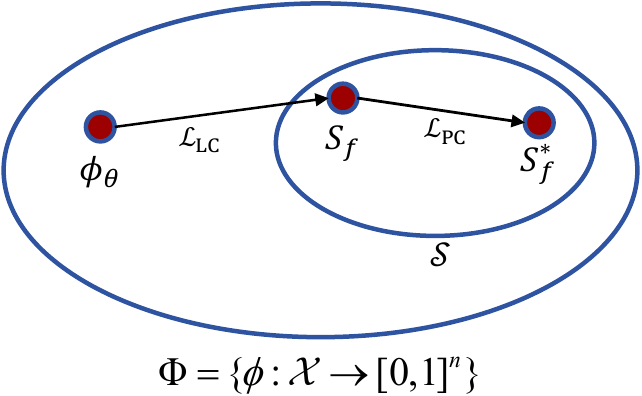}
  }
  \subfloat[Loss and $\alpha$]
  {
      \label{loss}\includegraphics[width=0.5\columnwidth]{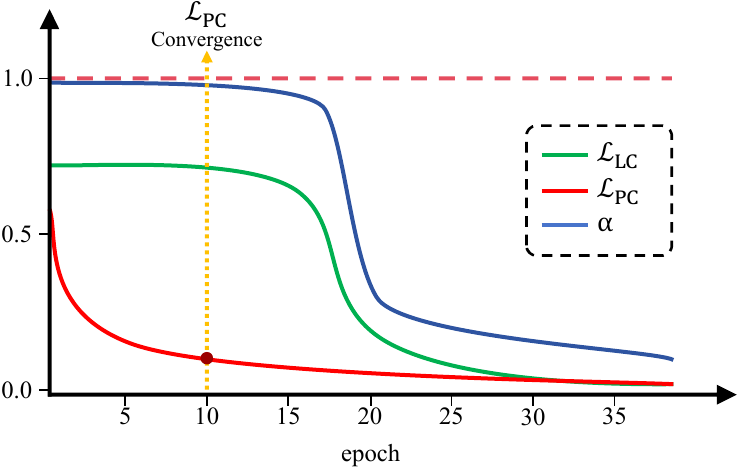}
  }
\caption{The function spaces $\Phi$ and $\mathcal{S}$ along with the dominant loss during the optimization of $\phi_\theta$ (a), and evolution of the weight $\alpha$ and the two loss terms during training (b).}
\label{layerResult}
\end{figure}

Let $\Phi = \{\phi: \mathcal{X} \to [0,1]^n\}$ denote the space of mappings from model inputs to $n$-dimensional vectors bounded in $[0,1]$. The family of saliency explanation mappings $\mathcal{S} = \{S_f: \mathcal{X} \to [0,1]^n\}$ (e.g., $\{S_f: \forall x, \mathrm{FC}(S_f(x); x, f) \ge 0.5\}$) forms a subset of $\Phi$, as illustrated in Figure \ref{funcSpace}.

Since Eq. (\ref{optSaliency}) is analytically intractable, \textbf{DeepFaith} trains a deep neural network $\phi_\theta \in \Phi$ (a transformer encoder in our experiments), parameterized by $\theta$, to approximate $S_f^* \in \mathcal{S} \subset \Phi$. Given a sample set $\mathcal{D}=\{x^{(i)}\}_{i=1}^{|\mathcal{D}|}$, faithfulness can be optimized using the \underline{L}ocal \underline{C}orrelation loss $\mathcal{L}_{\mathrm{LC}}$:
\begin{align}
\label{faithObj}
&\mathcal{L}_{\mathrm{LC}}(\phi_\theta;\mathcal{D},f)\\
\nonumber
=&-\frac{1}{|\mathcal{D}|}\sum_{x\in \mathcal{D}}\tau \left [ \left ( {\textstyle \sum_{i\in\mathcal{I}}\phi_\theta(x)_i} \right )_{\mathcal{I}\subseteq[n]},\left( \Delta_{x,f}(\mathcal{I})\right)_{\mathcal{I}\subseteq[n]} \right ] ,
\end{align}
where $\Delta$ and $\tau$ are user-defined. Notably, the trained explainer no longer requires access to $f$ during inference, as its decision rationale is already embedded through optimizing $\mathcal{L}_{\mathrm{LC}}$.

\begin{figure*}[ht]
\centering
  \subfloat[Image]
  {
      \label{cv}\includegraphics[width=0.42\textwidth]{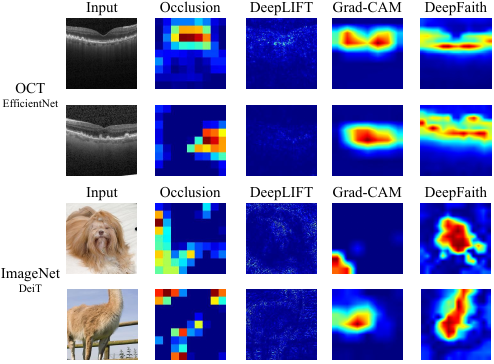}
  }
  \subfloat[Text]
  {
      \label{nlp}\includegraphics[width=0.37\textwidth]{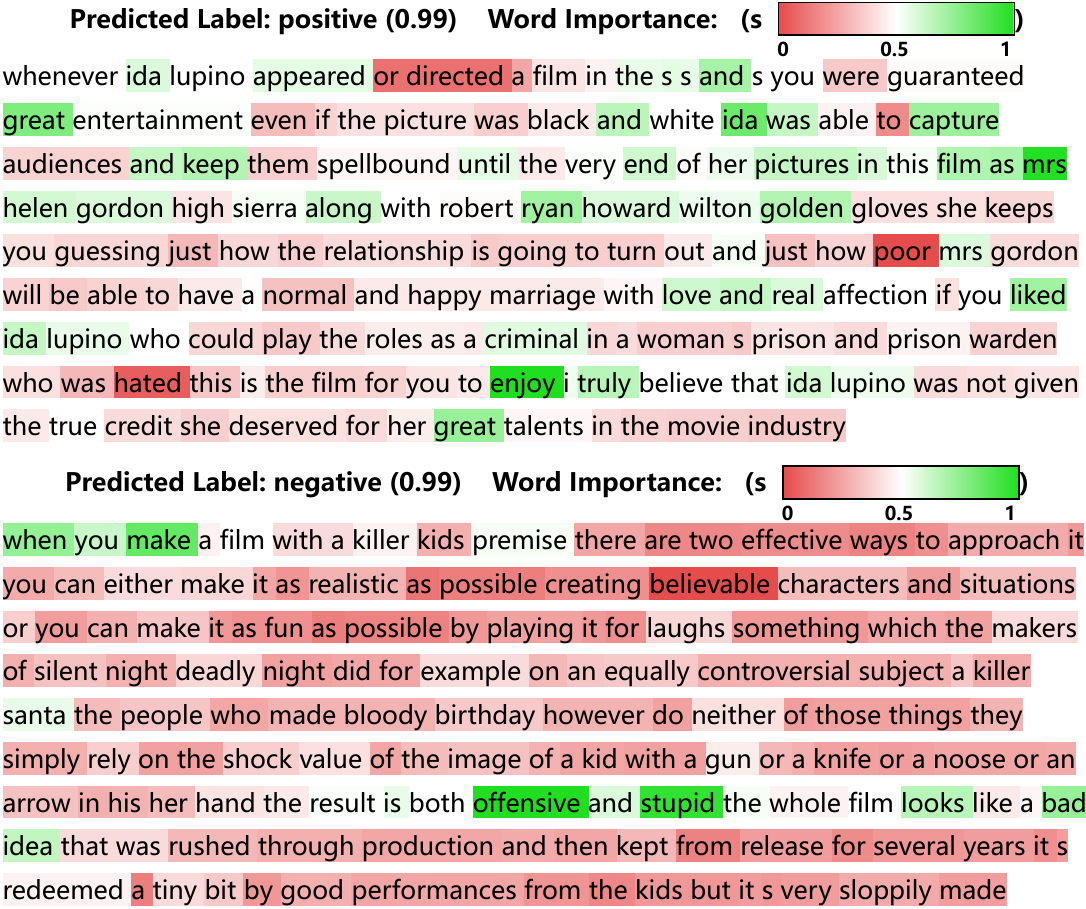}
  }
  \subfloat[Tabular]
  {
      \label{tab}\includegraphics[width=0.15\textwidth]{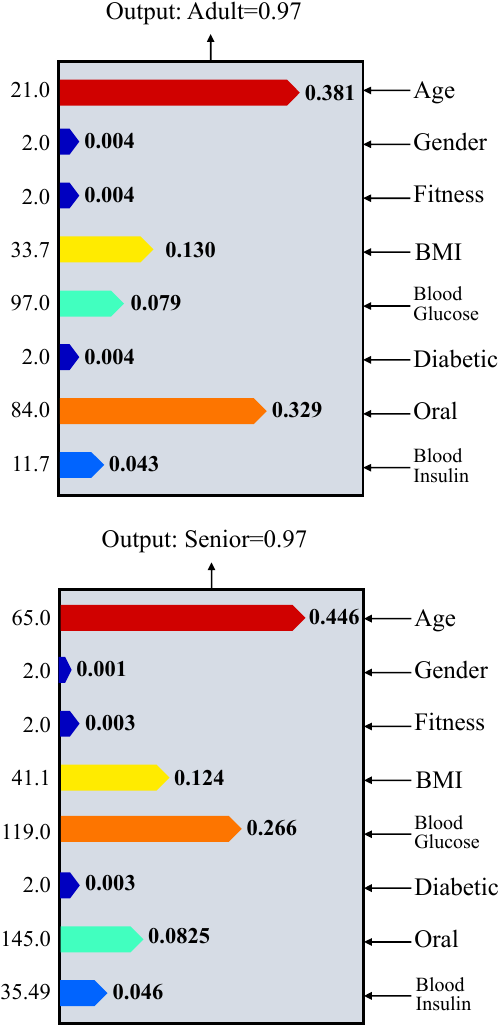}
  }
\caption{Explanations generated by the \textbf{DeepFaith} explainer trained on tasks across modalities, including comparisons with three other methods on image datasets (a), sentiment attributions for IMDb reviews (b), and contribution analysis of feature dimensions in NAP health data (c). See Appendix H for more visualizations.}
\label{visualization}
\end{figure*}

\section{Learning Framework of Faithful Explainer}
In this section, we propose, for the first time,  high-quality supervised explanation signals generation within our explainer learning framework. Explanations from different methods, although including domain-specific techniques and general-purpose algorithms, inherently reflect the functional dependency between input features and model predictions. Given that such patterns generalize across similar instances, an explainer can be trained to approximate the underlying mapping from inputs to saliency explanations.

We first generate a set of \textit{input–saliency explanation pairs} as supervised explanation signals illustrated in Figure \ref{deepfaith}. Given a sample set $\mathcal{D}$ and $K$ saliency explanation methods $\{S_f^{(i)}\}_{i=1}^K$ (e.g., Occlusion \cite{MatthewD.Zeiler1311}, Saliency \cite{simonyan2014saliency}, DeepLIFT \cite{Shrikumar1704}, Score-CAM \cite{wang2020score} and Grad-CAM++ \cite{Chattopadhyay1710}), we generate $K$ saliency explanation $\{S^{(j)}_f(x^{(i)})\}_{j=1}^K$ for each sample $x^{(i)}$. These explanations are then processed via deduplicating and filtering:
\begin{itemize}
    \item Deduplicating: We compute the pairwise cosine similarity between $K$ saliency explanations of a given sample $x^{(i)}$ and identify duplicate groups based on a manually defined similarity threshold. The first explanation in each group is retained, while the others are removed. After deduplicating, the number of distinct saliency explanations is denoted as $K_{\mathrm{dedup}}^{(i)} \le K$. This step aims to prevent highly similar explanations from introducing bias into the training of the explainer.
    \item Filtering: For each of the $K_{\mathrm{dedup}}^{(i)}$ retained explanations, we use all ten faithfulness metrics (the faithfulness of a saliency explanation can be evaluated from permutation perspective via $\mathfrak{P}$) to get their evaluation scores $(r_1,r_2,...,r_{10})$. We determine a filtering threshold $(\bar{r}_1,\bar{r}_2,...,\bar{r}_{10})$ by computing the $p$-quantile (or the $(1-p)$-quantile for metrics where lower is better) of all $K_{\mathrm{dedup}}^{(i)}$ scores under each metric. Finally, we retain $K_{\mathrm{filter}}^{(i)}\le K_{\mathrm{dedup}}^{(i)}$ explanations satisfying $\forall j\le 10,r_j\ge \bar{r}_j$ (or $r_j\le \bar{r}_j$ for metrics where lower is better).
\end{itemize}
After our explanation processing steps, the remained ones can be regarded as high-quality supervised explanation signals. For each input $x^{(i)}$, we replicate it $K_{\mathrm{filter}}^{(i)}$ times and pair each copy with its corresponding saliency explanation to construct the input–saliency explanation pair set
\begin{equation*}
\mathcal{Z}=\left \{ \left ( x^{(i)}, S_f^{(j)}(x^{(i)})\right )\mid i\le|\mathcal{D}|,j\in \left [ K_{\mathrm{filter}}^{(i)}\right ] \right \}.
\end{equation*}

\textbf{DeepFaith} optimizes the proximity between the explanations generated by $\phi_\theta$ and the high-quality saliency explanation through the \underline{P}attern \underline{C}onsistency loss $\mathcal{L}_{\mathrm{PC}}$:
\begin{equation}
\label{simObj}
\mathcal{L}_{\mathrm{PC}}(\phi_\theta;\mathcal{Z})=\frac{1}{|\mathcal{Z}|}\sum_{(x,s)\in\mathcal{Z}} \left ( 1-\tau \left [ \phi_\theta(x),s \right ] \right ),
\end{equation}
where $\tau$ can be any similarity measure and is not necessarily the same as the one used in Eq. (\ref{faithObj}). 

To jointly leverage and control both losses during training the explainer, \textbf{DeepFaith} introduces a weighting parameter $\alpha\in[0,1]$, forming the overall optimization \underline{OBJ}ective:
\begin{align}
\label{lossFinal}
&\mathcal{L}_{\mathrm{OBJ}}(\phi_\theta;\mathcal{D},f,\mathcal{Z})\\
\nonumber
=&\alpha \mathcal{L}_{\mathrm{PC}}(\phi_\theta;\mathcal{Z})+(1-\alpha)\mathcal{L}_{\mathrm{LC}}(\phi_\theta;\mathcal{D},f).
\end{align}
As shown in Figure \ref{loss}, at the early stage of training, we set $\alpha$ close to 1 (primarily optimizing $\mathcal{L}_{\mathrm{PC}}$), and gradually decrease it toward 0 after $\mathcal{L}_{\mathrm{PC}}$ loss convergence, shifting the focus to $\mathcal{L}_{\mathrm{LC}}$.

\begin{table*}[htbp]
  \centering
{\fontsize{9pt}{10pt}\selectfont
\setlength{\tabcolsep}{2pt}
  \begin{tabular}{l|ccc|ccc|cc|cc|c|c}
    \toprule
    \multirow{2}[2]{*}{{Method}} & \multicolumn{3}{c|}{OCT} & \multicolumn{3}{c|}{ImageNet} & \multicolumn{2}{c|}{IMDb} & \multicolumn{2}{c|}{AGNews} & NAP & WCD \\
    \cmidrule(lr){2-4} \cmidrule(lr){5-7} \cmidrule(lr){8-9} \cmidrule(lr){10-11} \cmidrule(lr){12-12} \cmidrule(lr){13-13}
     & DeiT & EfficientNet & ResNet & DeiT & EfficientNet & ResNet & LSTM & Transformer & LSTM & Transformer & MLP & MLP \\
    \midrule
   DeepFaith (ours) & 
    \textcolor{red}{\textbf{\ \ 3.4}} & 
    \textcolor{red}{\textbf{\ \ 2.9}} & 
    \textcolor{red}{\textbf{\ \ 4.1}} & 
    \textcolor{red}{\textbf{\ \ 4.4}} & 
    \textcolor{red}{\textbf{\ \ 4.4}} & 
    \textcolor{red}{\textbf{\ \ 3.3}} & 
    \textcolor{red}{\textbf{\ \ 2.3}} & 
    \textcolor{red}{\textbf{\ \ 2.1}} & 
    \textcolor{red}{\textbf{\ \ 2.9}} & 
    \textcolor{red}{\textbf{\ \ 2.7}} & 
    \textcolor{red}{\textbf{\ \ 1.8}} & 
    \textcolor{red}{\textbf{\ \ 1.8}} \\
    Integrated Grads & \ \ 7.8 & \ \ 7.6 & \ \ 4.8 & \ \ 6.4 & \ \ 7.0 & \ \ 5.4 & \ \ 3.3 & \ \ 5.6 & \ \ 4.9 & \ \ 5.9 & \ \ 2.8 & \ \ 5.2 \\
    Gradient SHAP & \textcolor{gray}{\ \ N/A} & \textcolor{gray}{\ \ N/A} & \textcolor{gray}{\ \ N/A} & \textcolor{gray}{\ \ N/A} & \textcolor{gray}{\ \ N/A} & \textcolor{gray}{\ \ N/A} & \ \ 4.4 & \ \ 4.0 & \ \ \textcolor{red}{\textbf{2.9}} & \ \ 4.2 & \ \ 4.7 & \ \ 7.3 \\
    DeepLIFT & \ \ 5.8 & \ \ 7.8 & \ \ 8.1 & \ \ 7.0 & \ \ 6.9 & \ \ 8.4 & \ \ 6.1 & \ \ 6.4 & \ \ 7.9 & \ \ 5.9 & \ \ 4.4 & \ \ 2.3 \\
    Saliency &  13.2 &  11.0 & 12.8 & 10.7 &  11.1 & 10.6 & \ \ 5.2 & \ \ 5.9 & \ \ 4.7 & \ \ 5.8 & \ \ 2.8 & \ \ 4.9 \\
    Occlusion & \ \ 8.5 & \ \ 6.5 & \ \ 8.4 & \ \ 8.9 & \ \ 9.6 & 10.9 & \ \ 4.6 & \ \ 3.6 & \ \ \textcolor{red}{\textbf{2.9}} & \ \ \textcolor{red}{\textbf{2.7}} & \ \ 3.3 & \ \ 5.9 \\
    Feature Ablation & \textcolor{gray}{\ \ N/A} & \textcolor{gray}{\ \ N/A} & \textcolor{gray}{\ \ N/A} & \textcolor{gray}{\ \ N/A} & \textcolor{gray}{\ \ N/A} & \textcolor{gray}{\ \ N/A} & \ \ 6.4 & \ \ 5.1 & \ \ 6.6 & \ \ 8.5 & \ \ 3.5 & \ \ 4.5 \\
    LIME &  12.3 & \ \ 8.1 & \ \ 9.9 &  10.7 & \ \ 6.6 & \ \ 8.5 & \ \ 7.7 & \ \ 6.8 & \ \ 4.6 & \ \ 4.5 & \ \ 4.7 & \ \ 2.7 \\
    Kernel SHAP & \ \ 4.2 & 10.9 &  12.1 & \ \ 7.0 & \ \ 5.9 & \ \ 8.9 & \ \ 5.0 & \ \ 5.5 & \ \ 6.4 & \ \ 3.9 & \ \ 3.9 & \ \ 8.9 \\
    Input × Gradient & \ \ 5.7 &  12.3 &  12.2 & \ \ 5.3 &  12.9 &  10.7 & \textcolor{gray}{\ \ N/A} & \textcolor{gray}{\ \ N/A} & \textcolor{gray}{\ \ N/A} & \textcolor{gray}{\ \ N/A} & \textcolor{gray}{\ \ N/A} & \textcolor{gray}{\ \ N/A} \\
    Guided Backprop &  12.3 & \ \ 6.5 & \ \ 7.6 &  11.4 &  10.3 &  10.4 & \textcolor{gray}{\ \ N/A} & \textcolor{gray}{\ \ N/A} & \textcolor{gray}{\ \ N/A} & \textcolor{gray}{\ \ N/A} & \textcolor{gray}{\ \ N/A} & \textcolor{gray}{\ \ N/A} \\
    Grad-CAM & \ \ 8.6 & \ \ 8.2 & \ \ 7.6 &  11.9 & \ \ 6.6 & \ \ 7.0 & \textcolor{gray}{\ \ N/A} & \textcolor{gray}{\ \ N/A} & \textcolor{gray}{\ \ N/A} & \textcolor{gray}{\ \ N/A} & \textcolor{gray}{\ \ N/A} & \textcolor{gray}{\ \ N/A} \\
    Score-CAM & \ \ 7.0 & \ \ 7.9 & \ \ 6.0 & \ \ 5.8 & \ \ 7.2 & \ \ 7.1 & \textcolor{gray}{\ \ N/A} & \textcolor{gray}{\ \ N/A} & \textcolor{gray}{\ \ N/A} & \textcolor{gray}{\ \ N/A} & \textcolor{gray}{\ \ N/A} & \textcolor{gray}{\ \ N/A} \\
    Grad-CAM++ & \ \ 5.0 &  10.0 & \ \ 7.4 & \ \ 4.9 & \ \ 7.8 & \ \ 8.3 & \textcolor{gray}{\ \ N/A} & \textcolor{gray}{\ \ N/A} & \textcolor{gray}{\ \ N/A} & \textcolor{gray}{\ \ N/A} & \textcolor{gray}{\ \ N/A} & \textcolor{gray}{\ \ N/A} \\
    Expected Grads & \ \ 6.9 & \ \ 8.0 & \ \ 7.1 & \ \ 7.1 & \ \ 9.5 & \ \ 6.4 & \textcolor{gray}{\ \ N/A} & \textcolor{gray}{\ \ N/A} & \textcolor{gray}{\ \ N/A} & \textcolor{gray}{\ \ N/A} & \textcolor{gray}{\ \ N/A} & \textcolor{gray}{\ \ N/A} \\
    DeepLIFT SHAP & \ \ 6.9 & \ \ 7.3 & \ \ 5.6 & \ \ 7.5 & \ \ 8.7 & \ \ 9.3 & \textcolor{gray}{\ \ N/A} & \textcolor{gray}{\ \ N/A} & \textcolor{gray}{\ \ N/A} & \textcolor{gray}{\ \ N/A} & \textcolor{gray}{\ \ N/A} & \textcolor{gray}{\ \ N/A} \\
    LRP &  12.0 & \ \ 4.5 & \ \ 5.4 &  10.2 & \ \ 5.0 & \ \ 4.5 & \textcolor{gray}{\ \ N/A} & \textcolor{gray}{\ \ N/A} & \textcolor{gray}{\ \ N/A} & \textcolor{gray}{\ \ N/A} & \textcolor{gray}{\ \ N/A} & \textcolor{gray}{\ \ N/A} \\
    \bottomrule
  \end{tabular}
  }
\caption{Comparison of average faithfulness between \textbf{DeepFaith} and other baseline methods across 12 explanation tasks. We report the average rank of each method under 10 faithfulness evaluation metrics, where \textcolor{red}{\textbf{Red}} denotes the optimal.}
\label{faithfulrank}
\end{table*}

This design aims to ensure that, in the early stages of training, the loss is dominated by $\mathcal{L}_{\mathrm{PC}}$, guiding the explainer to converge within the function space $\mathcal{S}$ shown in Figure \ref{funcSpace}, thereby acquiring basic explanatory capability. As training progresses, the benefit of optimizing $\mathcal{L}_{\mathrm{PC}}$ becomes limited by the signals. Therefore, we gradually decrease $\alpha$ to let $\mathcal{L}_{\mathrm{LC}}$ dominate the optimization, enabling the explainer to approximate $S_f^*$.

\section{Experiments}
In this section, we report the observations during the generation of the supervised explanation signals, as well as the performance and runtime of \textbf{DeepFaith} across various explanation tasks. We also provide ablation experiments to verify the necessity of combining $\mathcal{L}_{\rm PC}$ and $\mathcal{L}_{\rm LC}$.

\textbf{Experimental Setting:} To validate its \textit{domain-free} and \textit{model-agnostic} capabilities, \textbf{DeepFaith} is tested on image, text, and tabular modalities using various model architectures. These dataset-model combinations yield diverse settings with varying complexity, forming a \textit{comprehensive and challenging} benchmark for explanation quality. Dataset details are in Appendix B. All experiments were conducted on Ubuntu 22.04 with eight NVIDIA RTX A6000 GPUs.
\begin{itemize}
    \item \textbf{Image modality:} Following Latec \cite{Klein2024}, we use ImageNet \cite{Deng1000} and UCSD OCT Retina (OCT) \cite{Kermany2018}, explaining ResNet50 \cite{He1512}, EfficientNetb0 \cite{Tan1905}, and DeiT \cite{Touvron}.
    \item \textbf{Text modality:} IMDb Movie Review (IMDb) \cite{Maas2011} and AGNews \cite{Zhang2016} are used with LSTM and vanilla Transformer \cite{Vaswani2023}.
    \item \textbf{Tabular modality:} We use NHANES Age Prediction (NAP) \cite{nhanes_age_prediction_2019} and Wholesale Customers Data (WCD) \cite{wholesale_customers_2013} from UCI, with MLP-based predictors.
\end{itemize}
\begin{figure}[ht]
\centering
\includegraphics[width=1.0\columnwidth]{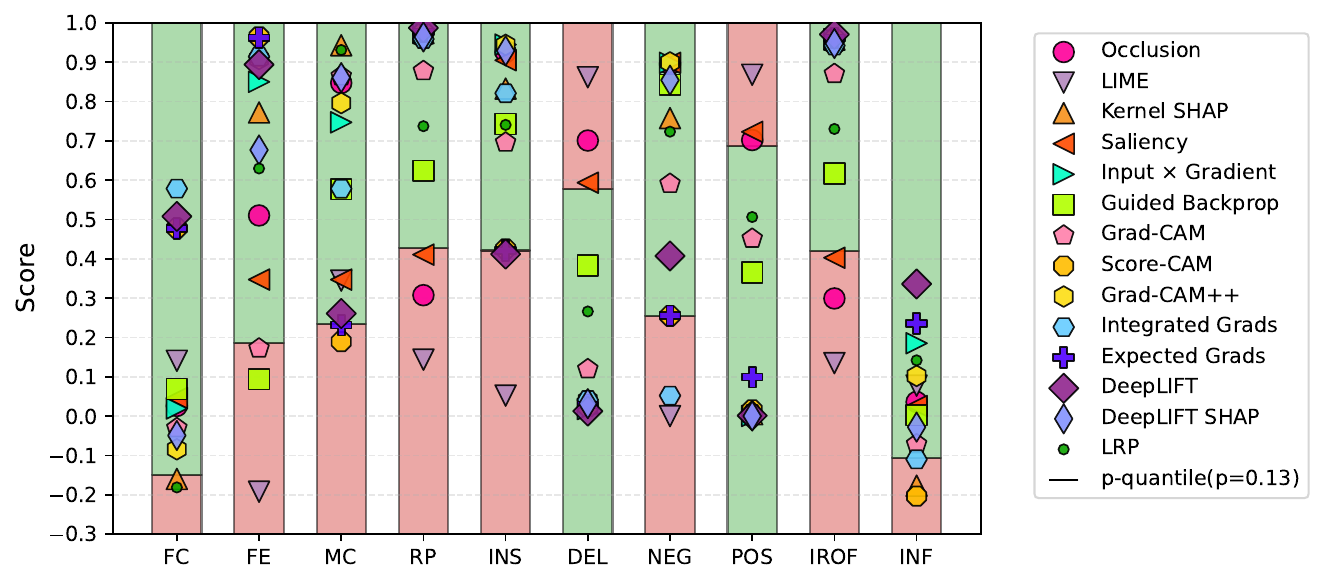}
\caption{In the OCT+DeiT explanation task, we compute 10 faithfulness metrics for each explanation method on a single sample and apply the $p$-quantile threshold to filter out low-quality explanations. Red and green regions denote the filtered-out range and retained range, respectively.}
\label{faithfulness}
\end{figure}

For the image modality, we generate supervised signals from and compare against the following baseline methods: Occlusion, LIME \cite{Ribeiro2016}, Kernel SHAP and DeepLIFT SHAP \cite{Lundberg1705}, Saliency, Input × Gradient \cite{Shrikumar1704}, Guided Backprop \cite{springenberg2015allconv}, Grad-CAM \cite{Selvaraju1610}, Score-CAM, Grad-CAM++, Integrated Grads \cite{Sundararajan1703}, Expected Grads \cite{erion2020}, DeepLIFT, and LRP \cite{binder2016lrp}. For the text and tabular modalities, we adopt Integrated Grads, Gradient SHAP \cite{Lundberg1705}, DeepLIFT, Saliency, Occlusion, Feature Ablation \cite{kokhlikyan2020captum}, LIME, and Kernel SHAP. Parameter settings are listed in Appendix C.

\begin{table*}[tbp]
  \centering
{\fontsize{9pt}{10pt}\selectfont
\setlength{\tabcolsep}{4pt}
    \begin{tabular}{l|c|rrrrcccccc}
    \toprule
    \multicolumn{1}{l}{Setting} & \multicolumn{1}{c}{Ablation} & \multicolumn{1}{r}{FC $\uparrow$} & \multicolumn{1}{r}{FE $\uparrow$} & \multicolumn{1}{r}{MC $\uparrow$} & \multicolumn{1}{r}{RP $\uparrow$} & \multicolumn{1}{r}{\ INS $\uparrow$} & \multicolumn{1}{r}{\ DEL $\downarrow$} & \multicolumn{1}{r}{\ NEG $\uparrow$} & \multicolumn{1}{r}{\ POS $\downarrow$} & \multicolumn{1}{r}{\ \ IROF $\uparrow$} & \multicolumn{1}{r}{\ INF $\uparrow$} \\
    \midrule
    \multirow{3}[1]{*}{OCT+DeiT} & $\mathcal{L}_{\rm OBJ}$  & \textcolor{red}{\textbf{\ \ 0.217}} & \textcolor{red}{\textbf{\ \ 0.475}} & \textcolor{red}{\textbf{\ \ 0.897}} & \textcolor{red}{\textbf{\ \ 0.643}} & \textcolor{red}{\textbf{0.944}} & \textcolor{red}{\textbf{0.356}} & \textcolor{red}{\textbf{0.917}} & \textcolor{red}{\textbf{0.368}} & \textcolor{red}{\textbf{0.638}} & \textcolor{red}{\textbf{0.089}} \\
          & $\mathcal{L}_{\rm PC\ \ }$    & 0.032 & 0.231 & 0.655 & 0.540 & 0.913 & 0.463 & 0.904 & 0.521 & 0.534 & 0.031 \\
          & $\mathcal{L}_{\rm LC\ \ }$    & 0.101 & 0.104 & 0.240 & 0.169 & 0.763 & 0.830 & 0.809 & 0.813 & 0.162 & 0.023 \\
    \midrule
    \multirow{3}[1]{*}{ImageNet+DeiT} & $\mathcal{L}_{\rm OBJ}$  & \textcolor{red}{\textbf{\ \ 0.026}} & \textcolor{red}{\textbf{\ \ 0.447}} & \textcolor{red}{\textbf{\ \ 0.884}} & \textcolor{red}{\textbf{\ \ 0.486}} & \textcolor{red}{\textbf{0.568}} & \textcolor{red}{\textbf{0.127}} & \textcolor{red}{\textbf{0.417}} & \textcolor{red}{\textbf{0.295}} & \textcolor{red}{\textbf{0.672}} & \textcolor{red}{\textbf{0.014}} \\
          & $\mathcal{L}_{\rm PC\ \ }$    & 0.022 & 0.364 & 0.823 & 0.456 & 0.501 & 0.185 & 0.406 & 0.366 & 0.638 & 0.008 \\
          & $\mathcal{L}_{\rm LC\ \ }$    & -0.047 & -0.051 & 0.033 & 0.373 & 0.552 & 0.380 & 0.397 & 0.414 & 0.493 & {-0.037\ } \\
    \midrule
    \multirow{3}[1]{*}{IMDb+Transformer} & $\mathcal{L}_{\rm OBJ}$  & \textcolor{red}{\textbf{\ \ 0.162}} & \textcolor{red}{\textbf{\ \ 0.495}} & \textcolor{red}{\textbf{\ \ 0.203}} & \textcolor{red}{\textbf{\ \ 0.759}} & \textcolor{red}{\textbf{0.806}} & \textcolor{red}{\textbf{0.189}} & \textcolor{red}{\textbf{0.799}} & \textcolor{red}{\textbf{0.205}} & \textcolor{red}{\textbf{0.742}} & \textcolor{red}{\textbf{0.047}} \\
          & $\mathcal{L}_{\rm PC\ \ }$    & 0.058 & 0.358 & 0.195 & 0.718 & 0.784 & 0.192 & 0.775 & 0.344 & 0.655 & 0.038 \\
          & $\mathcal{L}_{\rm LC\ \ }$    & 0.023 & 0.235 & 0.167 & 0.316 & 0.667 & 0.708 & 0.738 & 0.652 & 0.223 & 0.013 \\
    \midrule
    \multirow{3}[1]{*}{NAP+MLP} & $\mathcal{L}_{\rm OBJ}$  & \textcolor{red}{\textbf{0.788}} & \textcolor{red}{\textbf{0.763}} & \textcolor{red}{\textbf{0.952}} & \textcolor{red}{\textbf{0.957}} & \textcolor{red}{\textbf{0.844}} & \textcolor{red}{\textbf{0.031}} & \textcolor{red}{\textbf{0.770}} & \textcolor{red}{\textbf{0.031}} & \textcolor{red}{\textbf{0.844}} & \textcolor{red}{\textbf{0.238}} \\
          & $\mathcal{L}_{\rm PC\ \ }$    & 0.674 & 0.671 & 0.558 & 0.424 & 0.358 & 0.514 & 0.227 & 0.541 & 0.361 & 0.025 \\
          & $\mathcal{L}_{\rm LC\ \ }$    & 0.748 & 0.515 & 0.535 & 0.426 & 0.360 & 0.442 & 0.512 & 0.124 & 0.638 & 0.135 \\
    \bottomrule
    \end{tabular}%
    }
  \caption{Ablation study of \textbf{DeepFaith} on explanation tasks across different modalities. The table reports the average scores over ten faithfulness evaluation metrics, where $\mathcal{L}_{\rm OBJ}$ denotes the explainer trained with both loss terms.}
  \label{tabAblation}%
\end{table*}%

\subsection{Generating Supervised Explanation Signals}
Given a specific dataset and model, \textbf{DeepFaith} generates high-quality input-saliency explanation pairs before training. Taking the task of explaining DeiT’s predictions on ImageNet as an example, we use 14 widely adopted explanation methods from Captum \cite{kokhlikyan2020captum} to generate patch-level explanations for 20,000 validation samples. Each explanation is evaluated using 10 faithfulness metrics (detailed in Appendix D) from our unified framework.

Figure \ref{faithfulness} illustrates the faithfulness-based filtering process of the supervised explanation signals for one sample. For each evaluation metric, we compute the $p$-quantile and remove explanations deemed unfaithful by any of the metrics. Detailed processes for all explanation tasks are provided in Appendix E.

\subsection{Training Faithful Saliency Explainer}

We use a multi-layer Transformer Encoder as the explainer for its strength in processing sequential inputs. It encodes patch-based images, tokenized text, or tabular rows, followed by a normalized linear layer projecting to an $n$-dimensional saliency explanation. The weight $\alpha$ is scheduled as a sigmoid function of the epoch. Task-specific configurations are in Appendix F.

We split the supervised explanation signals into training and test sets and train the explainer. For each explanation task, we compare the faithfulness of \textbf{DeepFaith} against other baseline explanation methods. Each explanation is scored using all ten faithfulness metrics and averaged across all test samples (see Appendix G for full results). To concisely summarize the overall explanation quality of each method, we rank all explanation methods under each metric and report their average rankings.

Table \ref{faithfulrank} presents the evaluation results across all explanation tasks. \textbf{DeepFaith} consistently achieves the highest faithfulness, demonstrating that our method can generate higher-quality explanations than baseline methods across various modalities.

\begin{table}[htbp]
  \centering
{\fontsize{9pt}{10pt}\selectfont
\setlength{\tabcolsep}{1pt}
  \begin{tabular}{l|c|c|cc|c}
    \toprule
    \multirow{2}[2]{*}{{Method}} & \multicolumn{1}{c|}{ImageNet} & \multicolumn{1}{c|}{OCT} & \multicolumn{2}{c|}{AGNews} & \multicolumn{1}{c}{NAP}\\
    \cmidrule(lr){2-2} \cmidrule(lr){3-3} \cmidrule(lr){4-5} \cmidrule(lr){6-6} 
     & DeiT  & ResNet & LSTM & Transformer & MLP  \\
    \midrule
    DeepFaith & \textcolor{red}{\textbf{\ \ \ \ 3.103}} & \textcolor{red}{\textbf{\ \ \ \ 2.103}} & \textcolor{red}{\textbf{\ \ \ \ 1.217}} & \textcolor{red}{\textbf{\ \ \ \ 0.433}} & \textcolor{red}{\textbf{\ \ \ \ 0.117}} \\
    Integrated Grads & \ \ 95.132  & 103.721  & \ \ 53.941 & \ \ 58.473 & \ \ \ \ 2.839  \\
    DeepLIFT & \ \ 14.918 & \ \ 15.003 & \ \ \ \ 3.101 & \ \ \ \ 0.849 & \ \ \ \ 0.272  \\
    Saliency & \ \ 11.264  & \ \ \ \ 8.548 & \ \ \ \ 5.894 & \ \ \ \ 0.682 & \ \ \ \ 0.254  \\
    Occlusion & 115.435  & 170.348 & \ \ 61.725 & \ \ 25.734 & \ \ \ \ 0.563  \\
    LIME & 121.143   & \ \ 93.352 & \ \ 79.311 & 112.438 & \ \ 16.125 \\
    Kernel SHAP & \ \ 68.946  & \ \ 63.114 & \ \ 79.645 & 106.965 & \ \ 37.575 \\
    Grad-CAM & \ \ 13.756    & \ \ \ \ 9.617 & \textcolor{gray}{\ \ N/A} & \textcolor{gray}{\ \ N/A} & \textcolor{gray}{\ \ N/A}  \\
    Grad-CAM++ & \ \ \ \ 6.048   & \ \ \ \ 3.769 & \textcolor{gray}{\ \ N/A} & \textcolor{gray}{\ \ N/A} & \textcolor{gray}{\ \ N/A} \\
    Expected Grads & 124.935 & 122.261 & \textcolor{gray}{\ \ N/A} & \textcolor{gray}{\ \ N/A} & \textcolor{gray}{\ \ N/A} \\
    \bottomrule
  \end{tabular}
  }
  \caption{Average runtime (in ms) of \textbf{DeepFaith} and baseline methods for explaining a single sample.}
  \label{time}
\end{table}

\subsection{Visualization of DeepFaith Explanations}
Visualization bridges model predictions and human understanding, playing a key role in evaluating explanation methods. Figure \ref{visualization} illustrates explanations generated by \textbf{DeepFaith} across three modalities for well-trained models.

In Figure~\ref{cv}, we present two representative samples from the OCT and ImageNet datasets, along with visualizations from other methods. \textbf{DeepFaith}’s attributions are sharply focused on semantically meaningful regions with high visual clarity. Figure~\ref{nlp} shows two IMDb movie reviews predicted as positive and negative, with green highlights indicating the most influential words. \textbf{DeepFaith} emphasizes sentiment-consistent words in both reviews. In Figure~\ref{tab}, results on the NAP task show \textbf{DeepFaith} correctly attributes age as the dominant predictive feature rather than gender.

\subsection{Runtime Comparison}
Unlike classical post-hoc attribution methods that explain one instance at a time, \textbf{DeepFaith} incurs upfront costs for signal generation and explainer training. However, once trained, it serves as a high-performance explainer with comparable runtime, suitable for latency-critical scenarios such as stock trading and battlefield target acquisition.

Table \ref{time} reports the average per-sample explanation time (ms) across 5 tasks (full results in Appendix I). \textbf{DeepFaith} exhibits significantly lower latency than sampling-based methods like LIME, Kernel SHAP, and Occlusion, and also outperforms gradient-based methods such as Grad-CAM, Grad-CAM++, and Integrated Grads. This efficiency stems from its ability to decouple runtime from the architecture of the model being explained.

\subsection{Ablation Study}

We conduct ablation studies across all explanation tasks (full results in Appendix J) to evaluate the individual impact of each loss on \textbf{DeepFaith}’s performance. For each task, the explainer is trained for equal epochs under three settings: both losses, only $\mathcal{L}_{\rm PC}$, and only $\mathcal{L}_{\rm LC}$. We then report average faithfulness across all metrics.

Table \ref{tabAblation} presents results from 4 representative tasks, revealing a clear pattern: training with only $\mathcal{L}_{\rm PC}$ yields moderately faithful explanations but is limited by baseline methods, while using only $\mathcal{L}_{\rm LC}$ causes early optimization struggles and failure to converge. These outcomes align with our theoretical analysis.

\section{Conclusion}
\textbf{DeepFaith} is a domain-free and model-agnostic unified framework for training an explainer that leverages high-quality supervised explanation signals and theoretically grounded objectives to generate highly faithful explanations in a single forward pass. Moreover, it is \textit{highly extensible}: the baseline explanation methods employed to generate supervised signals can be substituted with any newly proposed techniques, whose processing can accommodate diverse engineering strategies; furthermore, the explainer architecture may comprise any deep neural network capable of handling sequential inputs. This flexibility suggests that \textbf{DeepFaith} has the potential to drive the emergence of a new paradigm for explainability, evolving alongside the development of the field.

\newpage


\section*{Supplementary material}

# Technical Appendix

# A  Related Work

Due to the fundamental challenge that ground truth explanations are inherently unavailable, existing explanations often rely on assumptions about how the model makes decisions. For example, SHAP assumes that prediction can be attributed through a feature-independent cooperative game; CAM assumes that the CNN ends with a global average pooling followed by a linear layer; and Integrated Grads assumes that important features correspond to large model gradients. However, such assumptions are frequently violated in real-world scenarios. Specifically, when features are highly correlated, their marginal contributions cannot be accurately identified by SHAP; CAM is inapplicable when the model architecture does not meet its structural requirements; and common operations such as ReLU and sigmoid may yield zero gradients for important features, rendering gradient-based methods ineffective. As a result, current explanation methods lack a general modeling of model attribution, making them inherently dependent on specific assumptions about the model's decision logic and architecture, which limits their practical utility. **DeepFaith** adopts a task-free and model-agnostic problem formulation, deriving a unified optimization objective from the underlying logic of faithfulness, thereby avoiding additional assumptions about the model being explained.

The field of interpretability has long pursued *learning to explain*, which aims to train a neural network to explain another model, enabling high-quality, real-time explanations via a single forward pass. L2X and CXPlain optimize the explainer using self-supervised objectives based on assumptions about model uncertainty and causal inference, respectively. ViT Shapley and L2E, on the other hand, construct datasets from existing explanation methods to learn mappings from inputs to explanations. However, self-supervised approaches suffer from overly broad hypothesis spaces and poor explainer initialization, leading to unstable loss convergence. Data-driven approaches often lack quality filtering, resulting in noisy training signals; moreover, since the explainer merely imitates existing explanations, its performance is inherently limited by the quality of the training data. **DeepFaith** constructs a high-quality explanation dataset through deduplication and filtering to guide the initial optimization of the explainer, and further enhances explanation faithfulness through a dedicated faithfulness-driven optimization objective.

# B  Datasets Information

| Dataset | Sample Num | Description |
| --- | --- | --- |
| ImageNet | 20,000 | ImageNet is a large-scale visual dataset encompassing real-world concepts such as animals, objects, and scenes, serving as a foundational dataset in the era of deep learning. |
| OCT | 1,000 | The OCT dataset comprises layered structural images of tissues such as the retina or cornea, used for medical image analysis and disease diagnosis, typically including high-resolution cross-sectional or 3D volumetric scan data. |
| AGNews | 127,600 | AGNews is a news article classification dataset containing English news headlines and content, covering four major categories (World, Sports, Business, Science), commonly used for text classification tasks. |
| IMDb | 50,000 | IMDb is a movie review sentiment analysis dataset containing English reviews labeled with binary sentiment (positive/negative). |
| NAP | 2,278 | The National Health and Nutrition Examination Survey (NHANES), administered by the Centers for Disease Control and Prevention, collects extensive health and nutritional information from a diverse U.S. population. |
| WCD | 440 | The data set refers to clients of a wholesale distributor. It includes the annual spending in monetary units (m.u.) on diverse product categories. |

Table 1: **DeepFaith** performs explanation tasks on datasets from three modalities. In NAP, the model predicts *age group* from body measurements; in WCD, it forecasts *Channel* using product sales data.

Table 1 presents **DeepFaith**'s explanation tasks on datasets from three different modalities. ImageNet is used for image recognition, OCT for medical image analysis, AGNews and IMDb for text classification and sentiment analysis, while NAP and WCD relate to health and customer behavior analysis. These datasets provide rich and representative testbeds for evaluating explanation methods.

## C  Baseline Explanation Methods

| Method | Parameters | Model Agnostic |
| --- | :---: | :---: |
| Occlusion | strides=25, sliding window=50 | True |
| LIME | num samples=10, perturbations per eval=5 | True |
| Kernel SHAP | num samples=10, perturbations per eval=5 | True |
| Saliency | None | True |
| Input × Gradient | None | True |
| Guided Backprop | None | True |
| Grad-CAM | None | False |
| Score-CAM | None | False |
| Grad-CAM++ | None | False |
| Integrated Grads | num steps=30, baselines=0 | True |
| Expected Grads | num samples=40, stdevs=0.001 | True |
| DeepLIFT | baselines=0, eps=1e-9 | True |
| DeepLIFT SHAP | None | True |
| LRP | eps=1e-4, gamma=0.25 | False |

Table 2: Parameters of baseline explanation methods for experiments on the image modality, as well as whether each method is model-agnostic.

| Method | Parameters | Model Agnostic |
| --- | :---: | :---: |
| Integrated Grads | num steps=20, baseline=0 | True |
| Gradient SHAP | num samples=5 | True |
| DeepLIFT | baselines=0, eps=1e-9 | True |
| Saliency | None | True |
| Occlusion | strides=1, sliding window=1 | True |
| Feature Ablation | perturbations per eval=1 | True |
| LIME | num samples=10, perturbations per eval=1 | True |
| Kernel SHAP | num samples=10, perturbations per eval=1 | True |

Table 3: Parameters of baseline explanation methods for experiments on text and tabular modality, as well as whether each method is model-agnostic.

Tables 2 and 3 detail the parameters for each baseline explanation method used across the three modalities. For Occlusion, the *sliding window* and *strides* must be specified. LIME and Kernel SHAP require setting the *num samples* and the number of perturbed features per evaluation. For Integrated Grads, *num steps* sets the discretization granularity along the integration path from the baseline to the input, while *baselines* defines the reference input from which integration starts. Expected Grads requires both the *num samples* and *stdevs*, the standard deviation of Gaussian noise added to each sampled baseline. DeepLIFT attributes contributions relative to *baselines* input, with *eps* preventing division-by-zero. In LRP, *eps* serves the same numerical stability purpose, and *gamma* adjusts the amplification of positive contributions. Gradient SHAP only requires specifying *num samples*, while Feature Ablation requires defining the number of perturbed features. For Saliency, Input × Gradient, Guided Backprop, Grad-CAM, Score-CAM, Grad-CAM++, and DeepLIFT SHAP, we adopt the default parameters provided by Captum.

# D  Faithfulness Evaluation Metrics

| Function | Variant | Formula | Description |
|---|---|---|---|
| $\Delta[f(x), f(x \setminus \mathcal{I})]$ | $\Delta_{\text{minus}}$ | $f(x) - f(x \setminus \mathcal{I})$ | Difference in model prediction before and after perturbation. |
| | $\Delta_{\text{variance}}$ | $\text{Var}\left[(f(x) - f(x \setminus \mathcal{I}_i))_{i=1}^N\right]$ | Variance of the predicted class score under perturbations. |
| $\Delta^-[f(x), f(x \setminus \mathcal{I})]$ | $\Delta^-_{\text{target}}$ | $f(x \setminus \mathcal{I})$ | Prediction value retained after perturbation. |
| | $\Delta^-_{\text{ratio}}$ | $f(x \setminus \mathcal{I})/f(x)$ | Ratio of prediction values before and after perturbation. |
| $\tau[(a_i)_{i=1}^N, (b_i)_{i=1}^N]$ | $\tau_{\text{pearson}}$ | $(a - \bar{a})^\top (b - \bar{b})/\|a - \bar{a}\|_2 \|b - \bar{b}\|_2$ | Classical Pearson correlation coefficient. |
| | $\tau_{\text{spearman}}$ | $\tau_{\text{pearson}}[\mathfrak{P}(a), \mathfrak{P}(b)]$ | Classical Spearman correlation. coefficient |
| | $\tau_{\text{mse}}$ | $1 - \frac{1}{2N}\sum_{i=1}^N (a_i - b_i)^2$ | Classical mean squared error. |

Table 4: Variants of the perturbation effect $\Delta$, presevation effect $\Delta^-$, and correlation $\tau$, along with their corresponding formulations and descriptions.

Table 4 reports several specific forms of the perturbation effect $\Delta$, preservation effect $\Delta^-$, and correlation $\tau$ used in our faithfulness evaluation methods. Note that we assume perturbing the originally predicted features will lead to a decrease in the model output; therefore, $\Delta^-_{\text{ratio}} \in [0, 1]$.

| Metric | $\Delta$ | $\Delta^-$ | $\tau$ | $|\mathcal{I}_i|$ | $N$ |
|---|---|---|---|---|---|
| FC | $\Delta_{\text{minus}}$ | – | $\tau_{\text{pearson}}$ | 3136 | 50 |
| FE | $\Delta_{\text{minus}}$ | – | $\tau_{\text{pearson}}$ | – | – |
| MC | $\Delta_{\text{variance}}$ | – | $\tau_{\text{spearman}}$ | – | – |
| RP | $\Delta_{\text{minus}}$ | – | – | – | – |
| INS | – | $\Delta^-_{\text{target}}$ | – | 3136 | – |
| DEL | – | $\Delta^-_{\text{target}}$ | – | 3136 | – |
| NEG | – | $\Delta^-_{\text{target}}$ | – | 3136 | – |
| POS | – | $\Delta^-_{\text{target}}$ | – | 3136 | – |
| IROF | – | $\Delta^-_{\text{ratio}}$ | – | 3136 | – |
| INF | $\Delta_{\text{minus}}$ | – | $\tau_{\text{pearson}}$ | – | 60 |

Table 5: Parameters of ten faithfulness metrics for three explanation tasks on ImageNet.

| Metric | $\Delta$ | $\Delta^-$ | $\tau$ | $|\mathcal{I}_i|$ | $N$ |
|---|---|---|---|---|---|
| FC | $\Delta_{\text{minus}}$ | – | $\tau_{\text{pearson}}$ | 3136 | 100 |
| FE | $\Delta_{\text{minus}}$ | – | $\tau_{\text{pearson}}$ | – | – |
| MC | $\Delta_{\text{variance}}$ | – | $\tau_{\text{spearman}}$ | – | – |
| RP | $\Delta_{\text{minus}}$ | – | – | – | – |
| INS | – | $\Delta^-_{\text{target}}$ | – | 3136 | – |
| DEL | – | $\Delta^-_{\text{target}}$ | – | 3136 | – |
| NEG | – | $\Delta^-_{\text{target}}$ | – | 3136 | – |
| POS | – | $\Delta^-_{\text{target}}$ | – | 3136 | – |
| IROF | – | $\Delta^-_{\text{ratio}}$ | – | 3136 | – |
| INF | $\Delta_{\text{minus}}$ | – | $\tau_{\text{pearson}}$ | – | 100 |

Table 6: Parameters of ten faithfulness metrics for three explanation tasks on OCT.

| Metric | $\Delta$ | $\Delta^-$ | $\tau$ | $|\mathcal{I}_i|$ | $N$ |
|---|---|---|---|---|---|
| FC | $\Delta_{\text{minus}}$ | – | $\tau_{\text{pearson}}$ | 10 | 100 |
| FE | $\Delta_{\text{minus}}$ | – | $\tau_{\text{pearson}}$ | – | – |
| MC | $\Delta_{\text{variance}}$ | – | $\tau_{\text{spearman}}$ | – | – |
| RP | $\Delta_{\text{minus}}$ | – | – | – | – |
| INS | – | $\Delta^-_{\text{target}}$ | – | 10 | – |
| DEL | – | $\Delta^-_{\text{target}}$ | – | 10 | – |
| NEG | – | $\Delta^-_{\text{target}}$ | – | 10 | – |
| POS | – | $\Delta^-_{\text{target}}$ | – | 10 | – |
| IROF | – | $\Delta^-_{\text{ratio}}$ | – | 10 | – |
| INF | $\Delta_{\text{minus}}$ | – | $\tau_{\text{pearson}}$ | – | 100 |

Table 7: Parameters of ten faithfulness metrics for four explanation tasks on text modality.

| Metric | $\Delta$ | $\Delta^-$ | $\tau$ | $|\mathcal{I}_i|$ | $N$ |
|---|---|---|---|---|---|
| FC | $\Delta_{\text{minus}}$ | – | $\tau_{\text{pearson}}$ | 1 | 30 |
| FE | $\Delta_{\text{minus}}$ | – | $\tau_{\text{pearson}}$ | – | – |
| MC | $\Delta_{\text{variance}}$ | – | $\tau_{\text{spearman}}$ | – | – |
| RP | $\Delta_{\text{minus}}$ | – | – | – | – |
| INS | – | $\Delta^-_{\text{target}}$ | – | 1 | – |
| DEL | – | $\Delta^-_{\text{target}}$ | – | 1 | – |
| NEG | – | $\Delta^-_{\text{target}}$ | – | 1 | – |
| POS | – | $\Delta^-_{\text{target}}$ | – | 1 | – |
| IROF | – | $\Delta^-_{\text{ratio}}$ | – | 1 | – |
| INF | $\Delta_{\text{minus}}$ | – | $\tau_{\text{pearson}}$ | – | 30 |

Table 8: Parameters of ten faithfulness metrics for two explanation tasks on tabular modality.

As shown in Table 5, we use the same evaluation metric parameters for the three explanation tasks on ImageNet, with those for the three tasks on OCT given in Table 6. For the four explanation tasks in the text modality, we use the parameters in Table 7, while those for the two tasks in the tabular modality are listed in Table 8.

# E  Supervised Explanation Signals

| Threshold | OCT | | | ImageNet | | | AGNews | | IMDb | | NAP | WCD |
|---|---|---|---|---|---|---|---|---|---|---|---|---|
| | DeiT | EfficientNet | ResNet | DeiT | EfficientNet | ResNet | LSTM | Transformer | LSTM | Transformer | MLP | MLP |
| Similarity Threshold | 0.90 | 0.90 | 0.88 | 0.90 | 0.85 | 0.90 | 0.60 | 0.65 | 0.60 | 0.60 | 0.90 | 0.93 |
| P-Quantile | 0.13 | 0.13 | 0.14 | 0.13 | 0.14 | 0.13 | 0.20 | 0.22 | 0.21 | 0.19 | 0.15 | 0.15 |

Table 9: Similarity threshold and $p$-quantile used by **DeepFaith** for deduplicating and filtering supervised explanation signals.

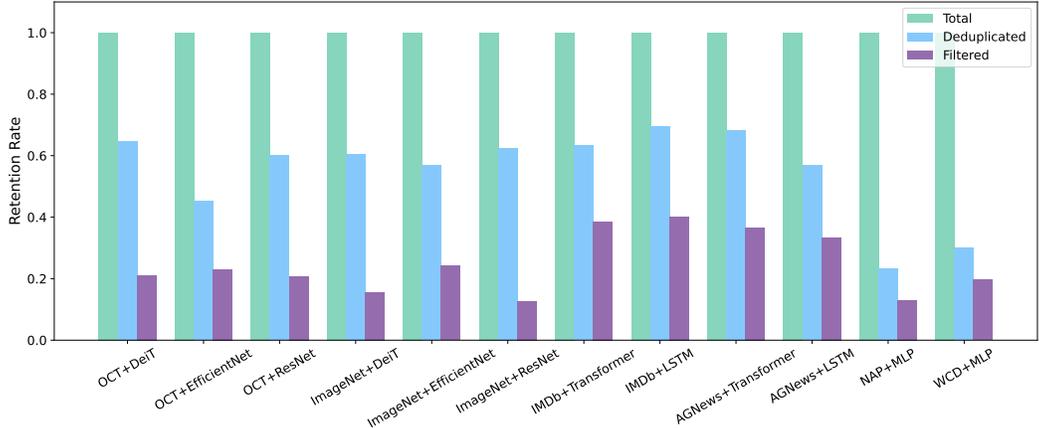

Figure 1: For the 12 explanation tasks we selected, **DeepFaith** performs deduplicating and filtering on the explanations generated by all baseline methods. *Total* denotes the total number of explanations, *Deduplicated* denotes the proportion remaining after deduplicating, and *Filtered* denotes the proportion remaining after further quality-based filtering on top of deduplicating.

Table 9 reports the similarity thresholds (for deduplicating) and $p$-quantiles (for filtering) used by **DeepFaith** when generating supervised explanation signals for 12 explanation tasks. Figure 1 illustrates, for each task, the proportion of explanations retained after deduplicating and filtering, relative to all generated explanations.

# F  Explainer Training Configurations

| Config | OCT | | | ImageNet | | | AGNews | | IMDb | | NAP | WCD |
|---|---|---|---|---|---|---|---|---|---|---|---|---|
| | DeiT | EfficientNet | ResNet | DeiT | EfficientNet | ResNet | LSTM | Transformer | LSTM | Transformer | MLP | MLP |
| seq len | 196 | 196 | 196 | 196 | 196 | 196 | 300 | 128 | 500 | 200 | 8 | 7 |
| embed dim | 768 | 768 | 768 | 768 | 768 | 768 | 256 | 512 | 512 | 512 | 64 | 64 |
| n head | 12 | 8 | 8 | 12 | 8 | 8 | 8 | 8 | 8 | 8 | 8 | 8 |
| max len | 1000 | 1000 | 1000 | 1000 | 1000 | 1000 | 1000 | 1000 | 1000 | 1000 | 100 | 100 |
| ffn hidden | 1024 | 1024 | 1024 | 2048 | 1024 | 1024 | 1024 | 1024 | 1024 | 1024 | 128 | 128 |
| n layers | 8 | 6 | 6 | 12 | 6 | 6 | 8 | 8 | 8 | 8 | 4 | 4 |
| drop prob | 0.01 | 0.5 | 0.01 | 0.01 | 0.05 | 0.05 | 0.01 | 0.01 | 0.01 | 0.01 | 0.01 | 0.01 |
| epochs | 100 | 100 | 100 | 100 | 100 | 100 | 50 | 50 | 50 | 50 | 20 | 20 |
| lr | 1e-5 | 1e-5 | 1e-5 | 1e-5 | 1e-5 | 1e-5 | 1e-5 | 1e-5 | 1e-5 | 1e-5 | 1e-5 | 1e-5 |
| weight decay | 1e-5 | 1e-5 | 1e-5 | 1e-5 | 1e-5 | 1e-5 | 1e-5 | 1e-5 | 1e-5 | 1e-5 | 1e-5 | 1e-5 |
| nr runs | 10 | 10 | 10 | 10 | 10 | 10 | 20 | 20 | 20 | 20 | 100 | 100 |

Table 10: The hyperparameters and training parameters of the explainer model (a multi-layer Transformer encoder) are set with varying complexities for different explanation tasks to ensure generalization. *nr runs* denotes the number of samples used for our local consistency loss.

# G  Faithfulness Comparison

In this section, we report the faithfulness comparison of **DeepFaith** and other baseline explanation methods across 12 explanation tasks, including the average over all test samples for each faithfulness evaluation metric and the average ranking across all metrics, including FC, FE, INF, and MC for saliency explanations, as well as DEL and INS, NEG and POS, RP, and IROF for permutation explanations.

| Method | FC | FE | MC | RP | INS | DEL | NEG | POS | IROF | INF | Mean Rank |
|---|---|---|---|---|---|---|---|---|---|---|---|
| DeepFaith | 0.217 | 0.475 | 0.897 | 0.643 | 0.944 | 0.356 | 0.917 | 0.368 | 0.638 | 0.089 | 3.4 |
| Occlusion | 0.034 | 0.602 | 0.850 | 0.520 | 0.903 | 0.479 | 0.924 | 0.486 | 0.517 | 0.009 | 8.5 |
| LIME | -0.015 | 0.132 | 0.394 | 0.531 | 0.689 | 0.466 | 0.650 | 0.469 | 0.527 | 0.024 | 12.3 |
| Kernel SHAP | 0.018 | 0.649 | 0.928 | 0.683 | 0.888 | 0.314 | 0.853 | 0.300 | 0.678 | 0.038 | 4.2 |
| Saliency | 0.006 | 0.105 | 0.287 | 0.314 | 0.740 | 0.685 | 0.764 | 0.691 | 0.310 | 0.004 | 13.2 |
| Input x Gradient | -0.002 | 0.575 | 0.759 | 0.604 | 0.920 | 0.395 | 0.887 | 0.312 | 0.598 | 0.030 | 5.7 |
| Guided Backprop | 0.009 | 0.075 | 0.424 | 0.339 | 0.690 | 0.660 | 0.774 | 0.682 | 0.336 | 0.026 | 12.3 |
| Grad-CAM | 0.030 | 0.464 | 0.784 | 0.582 | 0.877 | 0.414 | 0.869 | 0.374 | 0.577 | -0.008 | 8.6 |
| Score-CAM | 0.479 | 0.663 | 0.154 | 0.597 | 0.769 | 0.399 | 0.709 | 0.390 | 0.591 | 0.079 | 7.0 |
| Grad-CAM++ | 0.047 | 0.571 | 0.822 | 0.599 | 0.921 | 0.397 | 0.874 | 0.358 | 0.594 | 0.058 | 5.0 |
| Integrated Grads | 0.272 | 0.788 | 0.438 | 0.547 | 0.858 | 0.449 | 0.664 | 0.429 | 0.542 | 0.114 | 7.8 |
| Expected Grads | 0.439 | 0.724 | 0.214 | 0.597 | 0.761 | 0.399 | 0.701 | 0.377 | 0.591 | 0.060 | 6.9 |
| DeepLIFT | 0.322 | 0.696 | 0.182 | 0.610 | 0.784 | 0.384 | 0.762 | 0.391 | 0.605 | 0.165 | 5.8 |
| DeepLIFT SHAP | 0.060 | 0.530 | 0.810 | 0.584 | 0.908 | 0.413 | 0.866 | 0.390 | 0.579 | 0.043 | 6.9 |
| LRP | -0.010 | 0.012 | 0.944 | 0.410 | 0.707 | 0.589 | 0.696 | 0.617 | 0.408 | 0.019 | 12.0 |

Table 11: Faithfulness comparison of **DeepFaith** and baseline methods on OCT+DeiT task.

| Method | FC | FE | MC | RP | INS | DEL | NEG | POS | IROF | INF | Mean Rank |
|---|---|---|---|---|---|---|---|---|---|---|---|
| DeepFaith | 0.060 | 0.784 | 0.959 | 0.759 | 0.572 | 0.240 | 0.339 | 0.161 | 0.747 | 0.029 | 2.9 |
| Occlusion | 0.106 | 0.501 | 0.525 | 0.639 | 0.845 | 0.360 | 0.825 | 0.312 | 0.633 | 0.022 | 6.5 |
| LIME | 0.098 | 0.463 | 0.624 | 0.636 | 0.841 | 0.364 | 0.867 | 0.391 | 0.631 | -0.001 | 8.1 |
| Kernel SHAP | -0.031 | 0.054 | -0.044 | 0.691 | 0.403 | 0.303 | 0.252 | 0.272 | 0.685 | -0.007 | 10.9 |
| Saliency | -0.007 | 0.021 | 0.106 | 0.564 | 0.528 | 0.434 | 0.392 | 0.289 | 0.559 | 0.021 | 11.0 |
| Input x Gradient | 0.011 | -0.181 | -0.618 | 0.452 | 0.531 | 0.551 | 0.397 | 0.479 | 0.448 | -0.008 | 12.3 |
| Guided Backprop | 0.107 | 0.542 | 0.710 | 0.634 | 0.853 | 0.365 | 0.915 | 0.382 | 0.629 | 0.007 | 6.5 |
| Grad-CAM | 0.041 | 0.498 | 0.248 | 0.731 | 0.353 | 0.270 | 0.287 | 0.196 | 0.72 | -0.033 | 8.2 |
| Score-CAM | 0.015 | 0.516 | 0.627 | 0.728 | 0.406 | 0.271 | 0.162 | 0.196 | 0.718 | -0.018 | 7.9 |
| Grad-CAM++ | 0.007 | 0.099 | 0.363 | 0.710 | 0.350 | 0.295 | 0.287 | 0.224 | 0.702 | -0.019 | 10.0 |
| Integrated Grads | 0.074 | 0.522 | 0.116 | 0.735 | 0.347 | 0.259 | 0.282 | 0.176 | 0.725 | -0.051 | 7.6 |
| Expected Grads | -0.004 | 0.390 | 0.056 | 0.737 | 0.379 | 0.262 | 0.245 | 0.175 | 0.726 | 0.001 | 8.0 |
| DeepLIFT | 0.092 | 0.471 | 0.599 | 0.638 | 0.847 | 0.362 | 0.858 | 0.373 | 0.632 | -0.008 | 7.8 |
| DeepLIFT SHAP | 0.050 | 0.495 | 0.228 | 0.751 | 0.351 | 0.260 | 0.303 | 0.187 | 0.740 | -0.026 | 7.3 |
| LRP | 0.065 | 0.726 | 0.981 | 0.727 | 0.727 | 0.270 | 0.490 | 0.171 | 0.717 | 0.001 | 4.5 |

Table 12: Faithfulness comparison of **DeepFaith** and baseline methods on OCT+EfficientNet task.

| Method | FC | FE | MC | RP | INS | DEL | NEG | POS | IROF | INF | Mean Rank |
|---|---|---|---|---|---|---|---|---|---|---|---|
| DeepFaith | 0.135 | 0.534 | 0.942 | 0.744 | 0.863 | 0.248 | 0.655 | 0.242 | 0.742 | 0.015 | 4.1 |
| Occlusion | 0.072 | 0.523 | 0.766 | 0.677 | 0.930 | 0.316 | 0.888 | 0.372 | 0.677 | 0.028 | 8.4 |
| LIME | 0.062 | 0.413 | 0.738 | 0.672 | 0.931 | 0.320 | 0.908 | 0.364 | 0.672 | -0.010 | 9.9 |
| Kernel SHAP | -0.029 | 0.318 | 0.900 | 0.651 | 0.734 | 0.343 | 0.490 | 0.409 | 0.645 | -0.029 | 12.1 |
| Saliency | 0.011 | 0.111 | 0.327 | 0.551 | 0.675 | 0.443 | 0.576 | 0.368 | 0.549 | -0.003 | 12.8 |
| Input x Gradient | 0.010 | 0.013 | 0.818 | 0.502 | 0.673 | 0.490 | 0.573 | 0.438 | 0.500 | 0.012 | 12.2 |
| Guided Backprop | 0.039 | 0.452 | 0.769 | 0.685 | 0.932 | 0.307 | 0.922 | 0.333 | 0.686 | 0.008 | 7.6 |
| Grad-CAM | 0.139 | 0.638 | 0.375 | 0.730 | 0.667 | 0.261 | 0.457 | 0.148 | 0.723 | 0.007 | 7.6 |
| Score-CAM | 0.108 | 0.697 | 0.810 | 0.717 | 0.939 | 0.275 | 0.659 | 0.236 | 0.710 | -0.001 | 6.0 |
| Grad-CAM++ | 0.100 | 0.619 | 0.256 | 0.728 | 0.898 | 0.266 | 0.544 | 0.184 | 0.720 | 0.015 | 7.4 |
| Integrated Grads | 0.155 | 0.730 | 0.358 | 0.741 | 0.692 | 0.250 | 0.509 | 0.136 | 0.733 | 0.024 | 4.8 |
| Expected Grads | 0.158 | 0.661 | 0.368 | 0.731 | 0.668 | 0.260 | 0.465 | 0.135 | 0.724 | -0.010 | 7.1 |
| DeepLIFT | 0.075 | 0.425 | 0.771 | 0.683 | 0.932 | 0.308 | 0.908 | 0.357 | 0.683 | -0.002 | 8.1 |
| DeepLIFT SHAP | 0.103 | 0.696 | 0.351 | 0.737 | 0.808 | 0.253 | 0.582 | 0.152 | 0.732 | 0.015 | 5.6 |
| LRP | 0.109 | 0.684 | 0.852 | 0.724 | 0.932 | 0.268 | 0.727 | 0.242 | 0.718 | 0.011 | 5.4 |

Table 13: Faithfulness comparison of **DeepFaith** and baseline methods on OCT+ResNet task.

| Method | FC | FE | MC | RP | INS | DEL | NEG | POS | IROF | INF | Mean Rank |
|---|---|---|---|---|---|---|---|---|---|---|---|
| DeepFaith | 0.060 | 0.784 | 0.959 | 0.759 | 0.572 | 0.240 | 0.339 | 0.161 | 0.747 | 0.029 | 2.9 |
| Occlusion | 0.106 | 0.501 | 0.525 | 0.639 | 0.845 | 0.360 | 0.825 | 0.312 | 0.633 | 0.022 | 6.5 |
| LIME | 0.098 | 0.463 | 0.624 | 0.636 | 0.841 | 0.364 | 0.867 | 0.391 | 0.631 | -0.001 | 8.1 |
| Kernel SHAP | -0.031 | 0.054 | -0.044 | 0.691 | 0.403 | 0.303 | 0.252 | 0.272 | 0.685 | -0.007 | 10.9 |
| Saliency | -0.007 | 0.021 | 0.106 | 0.564 | 0.528 | 0.434 | 0.392 | 0.289 | 0.559 | 0.021 | 11.0 |
| Input x Gradient | 0.011 | -0.181 | -0.618 | 0.452 | 0.531 | 0.551 | 0.397 | 0.479 | 0.448 | -0.008 | 12.3 |
| Guided Backprop | 0.107 | 0.542 | 0.710 | 0.634 | 0.853 | 0.365 | 0.915 | 0.382 | 0.629 | 0.007 | 6.5 |
| Grad-CAM | 0.041 | 0.498 | 0.248 | 0.731 | 0.353 | 0.270 | 0.287 | 0.196 | 0.720 | -0.033 | 8.2 |
| Score-CAM | 0.015 | 0.516 | 0.627 | 0.728 | 0.406 | 0.271 | 0.162 | 0.196 | 0.718 | -0.018 | 7.9 |
| Grad-CAM++ | 0.007 | 0.099 | 0.363 | 0.710 | 0.350 | 0.295 | 0.287 | 0.224 | 0.702 | -0.019 | 10.0 |
| Integrated Grads | 0.074 | 0.522 | 0.116 | 0.735 | 0.347 | 0.259 | 0.282 | 0.176 | 0.725 | -0.051 | 7.6 |
| Expected Grads | -0.004 | 0.390 | 0.056 | 0.737 | 0.379 | 0.262 | 0.245 | 0.175 | 0.726 | 0.001 | 8.0 |
| DeepLIFT | 0.092 | 0.471 | 0.599 | 0.638 | 0.847 | 0.362 | 0.858 | 0.373 | 0.632 | -0.008 | 7.8 |
| DeepLIFT SHAP | 0.050 | 0.495 | 0.228 | 0.751 | 0.351 | 0.260 | 0.303 | 0.187 | 0.740 | -0.026 | 7.3 |
| LRP | 0.065 | 0.726 | 0.981 | 0.727 | 0.727 | 0.270 | 0.490 | 0.171 | 0.717 | 0.001 | 4.5 |

Table 14: Faithfulness comparison of **DeepFaith** and baseline methods on ImageNet+DeiT task.

| Method | FC | FE | MC | RP | INS | DEL | NEG | POS | IROF | INF | Mean Rank |
|---|---|---|---|---|---|---|---|---|---|---|---|
| DeepFaith | 0.021 | 0.217 | 0.835 | 0.591 | 0.525 | 0.174 | 0.405 | 0.156 | 0.749 | 0.004 | 4.4 |
| Occlusion | 0.007 | -0.017 | -0.525 | 0.333 | 0.499 | 0.360 | 0.445 | 0.415 | 0.476 | 0.028 | 9.6 |
| LIME | -0.008 | 0.112 | 0.921 | 0.461 | 0.593 | 0.231 | 0.542 | 0.306 | 0.657 | 0.012 | 6.6 |
| Kernel SHAP | 0.003 | 0.102 | 0.928 | 0.535 | 0.554 | 0.158 | 0.400 | 0.245 | 0.772 | 0.019 | 5.9 |
| Saliency | -0.016 | -0.022 | -0.161 | 0.410 | 0.423 | 0.283 | 0.350 | 0.338 | 0.559 | 0.011 | 11.1 |
| Input × Gradient | -0.053 | -0.073 | -0.802 | 0.344 | 0.423 | 0.349 | 0.380 | 0.393 | 0.446 | -0.015 | 12.9 |
| Guided Backprop | -0.002 | -0.012 | 0.279 | 0.321 | 0.535 | 0.372 | 0.496 | 0.430 | 0.469 | 0.003 | 10.3 |
| Grad-CAM | -0.009 | 0.122 | 0.282 | 0.595 | 0.404 | 0.099 | 0.255 | 0.182 | 0.852 | 0.000 | 6.6 |
| Score-CAM | -0.005 | 0.103 | 0.389 | 0.559 | 0.466 | 0.135 | 0.291 | 0.233 | 0.808 | -0.005 | 7.2 |
| Grad-CAM++ | 0.004 | 0.087 | 0.234 | 0.538 | 0.457 | 0.156 | 0.295 | 0.241 | 0.782 | 0.014 | 7.8 |
| Integrated Grads | -0.016 | 0.192 | 0.287 | 0.590 | 0.397 | 0.103 | 0.274 | 0.186 | 0.848 | -0.030 | 7.0 |
| Expected Grads | -0.037 | 0.137 | 0.263 | 0.492 | 0.460 | 0.201 | 0.334 | 0.286 | 0.717 | -0.014 | 9.5 |
| DeepLIFT | -0.050 | 0.164 | 0.890 | 0.450 | 0.575 | 0.243 | 0.537 | 0.304 | 0.639 | 0.027 | 6.9 |
| DeepLIFT SHAP | -0.066 | 0.112 | 0.287 | 0.588 | 0.397 | 0.106 | 0.248 | 0.195 | 0.843 | -0.035 | 8.7 |
| LRP | 0.005 | 0.139 | 0.964 | 0.555 | 0.506 | 0.138 | 0.363 | 0.231 | 0.794 | 0.010 | 5.0 |

Table 15: Faithfulness comparison of **DeepFaith** and baseline methods on ImageNet+EfficientNet task.

| Method | FC | FE | MC | RP | INS | DEL | NEG | POS | IROF | INF | Mean Rank |
|---|---|---|---|---|---|---|---|---|---|---|---|
| DeepFaith | 0.031 | 0.254 | 0.938 | 0.677 | 0.577 | 0.106 | 0.471 | 0.182 | 0.871 | 0.019 | 3.3 |
| Occlusion | 0.023 | 0.053 | 0.125 | 0.457 | 0.507 | 0.285 | 0.393 | 0.366 | 0.633 | 0.015 | 10.9 |
| LIME | -0.017 | 0.202 | 0.942 | 0.573 | 0.562 | 0.169 | 0.498 | 0.280 | 0.776 | -0.004 | 8.5 |
| Kernel SHAP | -0.015 | 0.224 | 0.761 | 0.579 | 0.435 | 0.164 | 0.344 | 0.255 | 0.775 | 0.021 | 8.9 |
| Saliency | 0.006 | 0.030 | 0.005 | 0.627 | 0.365 | 0.115 | 0.268 | 0.220 | 0.840 | 0.003 | 10.6 |
| Input × Gradient | -0.041 | -0.012 | -0.337 | 0.640 | 0.373 | 0.101 | 0.257 | 0.204 | 0.856 | -0.041 | 10.7 |
| Guided Backprop | -0.012 | 0.066 | 0.414 | 0.492 | 0.513 | 0.250 | 0.407 | 0.340 | 0.668 | 0.028 | 10.4 |
| Grad-CAM | 0.042 | 0.199 | 0.569 | 0.655 | 0.367 | 0.088 | 0.271 | 0.199 | 0.870 | -0.026 | 7.0 |
| Score-CAM | 0.041 | 0.268 | 0.697 | 0.628 | 0.380 | 0.116 | 0.282 | 0.230 | 0.846 | 0.009 | 7.1 |
| Grad-CAM++ | 0.057 | 0.104 | 0.626 | 0.627 | 0.449 | 0.116 | 0.284 | 0.212 | 0.835 | -0.017 | 8.3 |
| Integrated Grads | 0.023 | 0.244 | 0.570 | 0.666 | 0.368 | 0.078 | 0.256 | 0.162 | 0.882 | 0.035 | 5.4 |
| Expected Grads | 0.012 | 0.253 | 0.560 | 0.654 | 0.352 | 0.090 | 0.259 | 0.180 | 0.872 | 0.045 | 6.4 |
| DeepLIFT | -0.001 | 0.181 | 0.910 | 0.510 | 0.526 | 0.232 | 0.456 | 0.327 | 0.693 | 0.040 | 8.4 |
| DeepLIFT SHAP | 0.019 | 0.178 | 0.604 | 0.629 | 0.362 | 0.114 | 0.267 | 0.207 | 0.839 | -0.003 | 9.3 |
| LRP | 0.036 | 0.345 | 0.946 | 0.648 | 0.450 | 0.096 | 0.374 | 0.186 | 0.868 | 0.004 | 4.5 |

Table 16: Faithfulness comparison of **DeepFaith** and baseline methods on ImageNet+ResNet task.

Tables 11, 12, and 13 present the experimental results for the three explanation tasks on the OCT dataset, while Tables 14, 15, and 16 report those on the ImageNet dataset.

In all cases, **DeepFaith** achieves the optimal average ranking, indicating that its generated explanations exhibit strong generalization from a faithfulness perspective across test set samples.

| Method | FC | FE | MC | RP | INS | DEL | NEG | POS | IROF | INF | Mean Rank |
|---|---|---|---|---|---|---|---|---|---|---|---|
| DeepFaith | 0.172 | 0.486 | 0.360 | 0.812 | 0.872 | 0.151 | 0.869 | 0.182 | 0.813 | 0.038 | 2.3 |
| Integrated Grads | 0.201 | 0.492 | 0.048 | 0.798 | 0.832 | 0.167 | 0.827 | 0.195 | 0.798 | 0.324 | 3.3 |
| Gradient SHAP | 0.081 | 0.277 | 0.054 | 0.796 | 0.832 | 0.168 | 0.828 | 0.207 | 0.797 | 0.278 | 4.4 |
| DeepLIFT | 0.031 | 0.088 | 0.154 | 0.517 | 0.840 | 0.442 | 0.846 | 0.501 | 0.521 | 0.081 | 6.1 |
| Saliency | 0.093 | 0.256 | 0.050 | 0.706 | 0.835 | 0.255 | 0.840 | 0.328 | 0.709 | 0.161 | 5.2 |
| Occlusion | 0.087 | 0.311 | 0.049 | 0.794 | 0.832 | 0.169 | 0.829 | 0.209 | 0.795 | 0.313 | 4.6 |
| Feature Ablation | 0.042 | 0.082 | 0.018 | 0.765 | 0.832 | 0.199 | 0.829 | 0.233 | 0.766 | 0.082 | 6.4 |
| LIME | 0.020 | 0.056 | -0.844 | 0.408 | 0.890 | 0.551 | 0.864 | 0.530 | 0.410 | 0.036 | 7.7 |
| Kernel SHAP | 0.094 | 0.250 | 0.994 | 0.452 | 0.913 | 0.508 | 0.900 | 0.411 | 0.453 | 0.060 | 5.0 |

Table 17: Faithfulness comparison of **DeepFaith** and baseline methods on IMDb+LSTM task.

| Method | FC | FE | MC | RP | INS | DEL | NEG | POS | IROF | INF | Mean Rank |
|---|---|---|---|---|---|---|---|---|---|---|---|
| DeepFaith | 0.162 | 0.495 | 0.203 | 0.759 | 0.806 | 0.189 | 0.799 | 0.205 | 0.742 | 0.047 | 2.1 |
| Integrated Grads | 0.061 | 0.208 | 0.037 | 0.634 | 0.766 | 0.307 | 0.762 | 0.355 | 0.625 | 0.128 | 5.6 |
| Gradient SHAP | 0.117 | 0.347 | 0.066 | 0.739 | 0.760 | 0.210 | 0.747 | 0.219 | 0.722 | 0.301 | 4.0 |
| DeepLIFT | 0.019 | 0.062 | 0.073 | 0.338 | 0.800 | 0.599 | 0.774 | 0.577 | 0.332 | 0.058 | 6.4 |
| Saliency | 0.050 | 0.167 | 0.060 | 0.577 | 0.770 | 0.363 | 0.768 | 0.406 | 0.570 | 0.107 | 5.9 |
| Occlusion | 0.119 | 0.352 | 0.071 | 0.740 | 0.760 | 0.210 | 0.747 | 0.218 | 0.723 | 0.308 | 3.6 |
| Feature Ablation | 0.110 | 0.324 | 0.060 | 0.729 | 0.760 | 0.220 | 0.748 | 0.236 | 0.712 | 0.258 | 5.1 |
| LIME | 0.021 | 0.064 | 0.075 | 0.299 | 0.817 | 0.638 | 0.767 | 0.592 | 0.294 | 0.040 | 6.8 |
| Kernel SHAP | 0.102 | 0.178 | 0.979 | 0.283 | 0.819 | 0.655 | 0.803 | 0.566 | 0.276 | 0.052 | 5.5 |

Table 18: Faithfulness comparison of **DeepFaith** and baseline methods on IMDb+Transformer task.

| Method | FC | FE | MC | RP | INS | DEL | NEG | POS | IROF | INF | Mean Rank |
|---|---|---|---|---|---|---|---|---|---|---|---|
| DeepFaith | 0.363 | 0.597 | 0.629 | 0.648 | 0.919 | 0.197 | 0.906 | 0.256 | 0.650 | 0.275 | 2.9 |
| Integrated Grads | 0.492 | 0.597 | 0.052 | 0.571 | 0.838 | 0.375 | 0.825 | 0.383 | 0.573 | 0.313 | 4.9 |
| Gradient SHAP | 0.302 | 0.623 | 0.019 | 0.787 | 0.879 | 0.161 | 0.868 | 0.176 | 0.786 | 0.434 | 2.9 |
| DeepLIFT | 0.032 | 0.095 | 0.077 | 0.373 | 0.846 | 0.568 | 0.825 | 0.564 | 0.377 | 0.074 | 7.9 |
| Saliency | 0.497 | 0.597 | 0.052 | 0.571 | 0.838 | 0.375 | 0.825 | 0.383 | 0.573 | 0.316 | 4.7 |
| Occlusion | 0.313 | 0.671 | -0.001 | 0.727 | 0.882 | 0.122 | 0.871 | 0.214 | 0.726 | 0.362 | 2.9 |
| Feature Ablation | 0.226 | 0.311 | 0.028 | 0.539 | 0.844 | 0.407 | 0.831 | 0.416 | 0.541 | 0.217 | 6.6 |
| LIME | 0.061 | 0.202 | 0.269 | 0.657 | 0.908 | 0.284 | 0.890 | 0.296 | 0.663 | 0.067 | 4.6 |
| Kernel SHAP | 0.112 | 0.245 | 0.952 | 0.445 | 0.853 | 0.496 | 0.838 | 0.498 | 0.449 | 0.111 | 6.4 |

Table 19: Faithfulness comparison of **DeepFaith** and baseline methods on AGNews+LSTM task.

| Method | FC | FE | MC | RP | INS | DEL | NEG | POS | IROF | INF | Mean Rank |
|---|---|---|---|---|---|---|---|---|---|---|---|
| DeepFaith | 0.111 | 0.318 | 0.464 | 0.663 | 0.901 | 0.277 | 0.864 | 0.194 | 0.651 | 0.082 | 2.7 |
| Integrated Grads | 0.119 | 0.290 | 0.015 | 0.400 | 0.808 | 0.545 | 0.776 | 0.529 | 0.370 | 0.154 | 5.9 |
| Gradient SHAP | 0.091 | 0.402 | 0.065 | 0.627 | 0.811 | 0.321 | 0.780 | 0.319 | 0.601 | 0.249 | 4.2 |
| DeepLIFT | 0.036 | 0.095 | 0.222 | 0.464 | 0.813 | 0.474 | 0.785 | 0.477 | 0.444 | 0.110 | 5.9 |
| Saliency | 0.121 | 0.290 | 0.015 | 0.400 | 0.808 | 0.545 | 0.776 | 0.529 | 0.370 | 0.154 | 5.8 |
| Occlusion | 0.099 | 0.500 | 0.028 | 0.787 | 0.811 | 0.163 | 0.782 | 0.163 | 0.755 | 0.290 | 2.7 |
| Feature Ablation | 0.027 | 0.204 | 0.019 | 0.391 | 0.775 | 0.554 | 0.739 | 0.557 | 0.363 | 0.099 | 8.5 |
| LIME | 0.081 | 0.231 | -0.038 | 0.655 | 0.816 | 0.285 | 0.789 | 0.282 | 0.633 | 0.175 | 4.5 |
| Kernel SHAP | 0.102 | 0.236 | 0.993 | 0.609 | 0.897 | 0.328 | 0.858 | 0.271 | 0.599 | 0.123 | 3.9 |

Table 20: Faithfulness comparison of **DeepFaith** and baseline methods on AGNews+Transformer task.

Tables 17 and 18 present the experimental results for the two explanation tasks on the IMDb dataset, while Tables 19 and 20 report those on the AGNews dataset. **DeepFaith** outperforms all baseline explanation methods on the two IMDb tasks and matches the overall faithfulness of the best-performing baseline on AGNews. These results demonstrate that **DeepFaith** also provides highly faithful explanations for data in the text modality.

| Method | FC | FE | MC | RP | INS | DEL | NEG | POS | IROF | INF | Mean Rank |
|---|---|---|---|---|---|---|---|---|---|---|---|
| DeepFaith | 0.788 | 0.763 | 0.952 | 0.957 | 0.844 | 0.031 | 0.770 | 0.031 | 0.844 | 0.238 | 1.8 |
| Integrated Grads | -0.173 | -0.186 | 0.077 | -0.052 | 0.821 | 0.875 | 0.823 | 0.875 | -0.064 | -0.154 | 2.8 |
| Gradient SHAP | -0.178 | -0.189 | 0.043 | -0.052 | 0.817 | 0.875 | 0.819 | 0.875 | -0.064 | -0.153 | 4.7 |
| DeepLift | -0.175 | -0.187 | -0.033 | -0.052 | 0.814 | 0.875 | 0.817 | 0.875 | -0.064 | -0.146 | 4.4 |
| Saliency | -0.173 | -0.186 | 0.077 | -0.052 | 0.821 | 0.875 | 0.823 | 0.875 | -0.064 | -0.154 | 2.8 |
| Occlusion | -0.171 | -0.188 | 0.132 | -0.052 | 0.816 | 0.875 | 0.819 | 0.875 | -0.064 | -0.151 | 3.3 |
| Feature Ablation | -0.173 | -0.188 | 0.098 | -0.052 | 0.816 | 0.875 | 0.819 | 0.875 | -0.064 | -0.151 | 3.5 |
| LIME | -0.175 | -0.189 | 0.052 | -0.052 | 0.814 | 0.875 | 0.817 | 0.875 | -0.064 | -0.148 | 4.7 |
| Kernel SHAP | -0.180 | -0.187 | 0.531 | -0.052 | 0.820 | 0.875 | 0.822 | 0.875 | -0.064 | -0.158 | 3.9 |

Table 21: Faithfulness comparison of **DeepFaith** and baseline methods on NAP+MLP task.

| Method | FC | FE | MC | RP | INS | DEL | NEG | POS | IROF | INF | Mean Rank |
|---|---|---|---|---|---|---|---|---|---|---|---|
| DeepFaith | 0.961 | 0.961 | 0.679 | 0.364 | 0.723 | 0.551 | 0.575 | 0.455 | 0.306 | 0.929 | 1.8 |
| Integrated Grads | 0.142 | 0.132 | 0.534 | 0.074 | 0.536 | 0.492 | 0.526 | 0.477 | 0.066 | 0.197 | 5.2 |
| Gradient SHAP | 0.169 | 0.130 | 0.523 | 0.070 | 0.534 | 0.496 | 0.522 | 0.488 | 0.063 | 0.202 | 7.3 |
| DeepLIFT | 0.184 | 0.152 | 0.666 | 0.089 | 0.548 | 0.478 | 0.539 | 0.467 | 0.081 | 0.218 | 2.3 |
| Saliency | 0.160 | 0.132 | 0.534 | 0.074 | 0.536 | 0.492 | 0.526 | 0.477 | 0.066 | 0.200 | 4.9 |
| Occlusion | 0.168 | 0.136 | 0.672 | 0.072 | 0.536 | 0.495 | 0.524 | 0.486 | 0.064 | 0.198 | 5.9 |
| Feature Ablation | 0.175 | 0.141 | 0.654 | 0.073 | 0.538 | 0.493 | 0.526 | 0.484 | 0.066 | 0.211 | 4.5 |
| LIME | 0.186 | 0.159 | 0.551 | 0.089 | 0.547 | 0.479 | 0.537 | 0.468 | 0.080 | 0.220 | 2.7 |
| Kernel SHAP | 0.071 | 0.049 | -0.056 | 0.048 | 0.515 | 0.516 | 0.500 | 0.512 | 0.042 | 0.181 | 8.9 |

Table 22: Faithfulness comparison of **DeepFaith** and baseline methods on WCD+MLP task.

Tables 21 and 22 present the experimental results of explaining MLPs on the UCI dataset. Many baseline explanation methods score poorly on correlation-based faithfulness metrics such as FC and FE, as their perturbation procedures are overly complex relative to the simplicity of the explained model. In contrast,**DeepFaith** directly improves correlation-based faithfulness by optimizing the local correlation loss, achieving comprehensive superiority over the baseline explanation methods.

## H More Visualizations

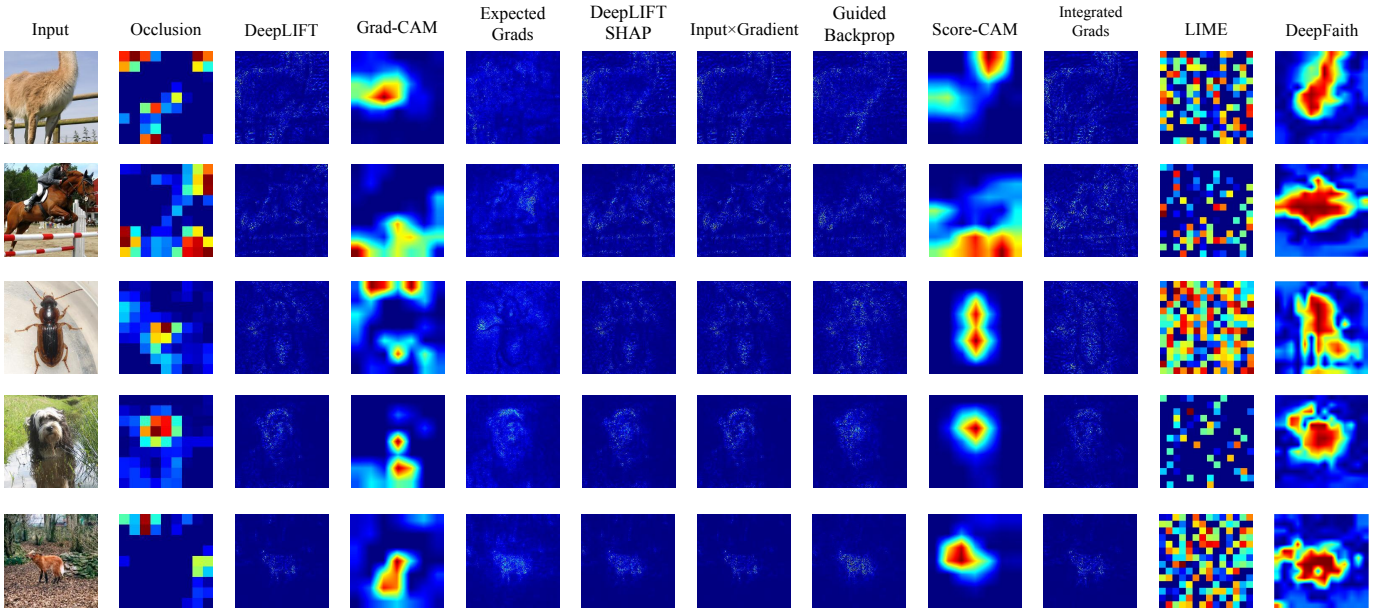

Figure 2: Saliency explanations generated by **DeepFaith** for DeiT on the ImageNet dataset, along with those produced by baseline methods for the same inputs.

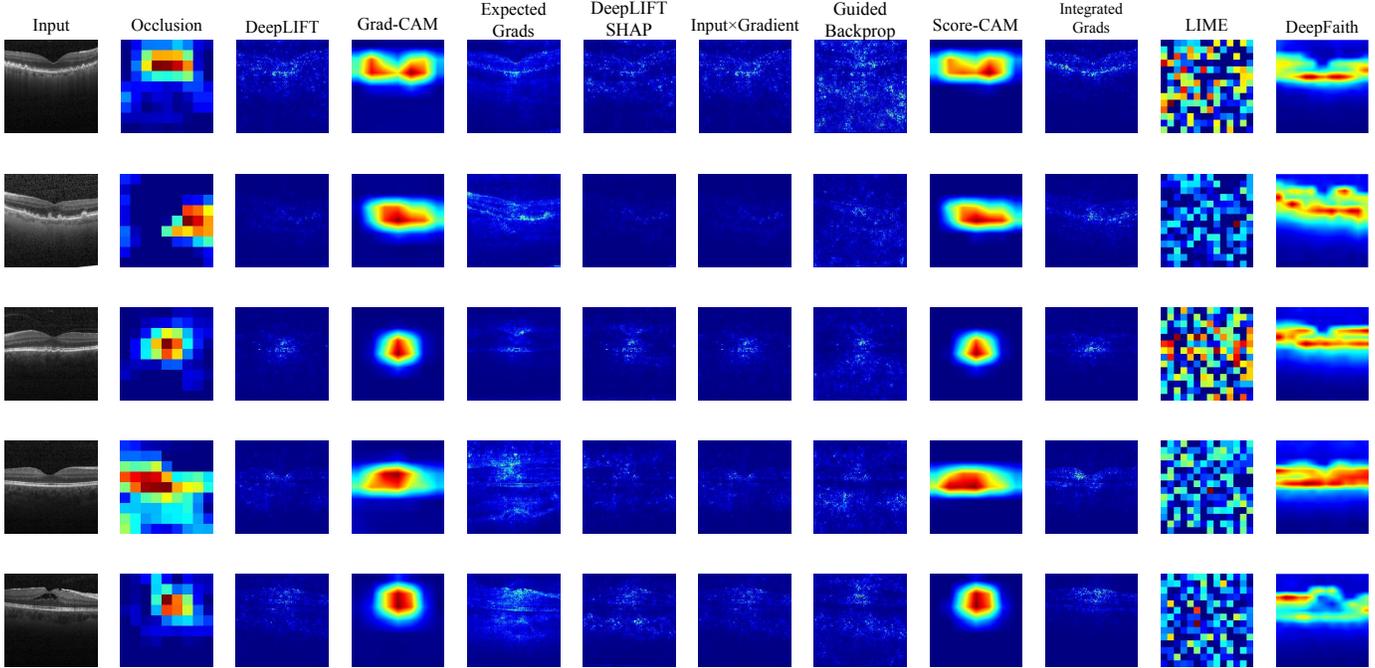

Figure 3: Saliency explanations generated by **DeepFaith** for EfficientNet on the OCT dataset, along with those produced by baseline methods for the same inputs.

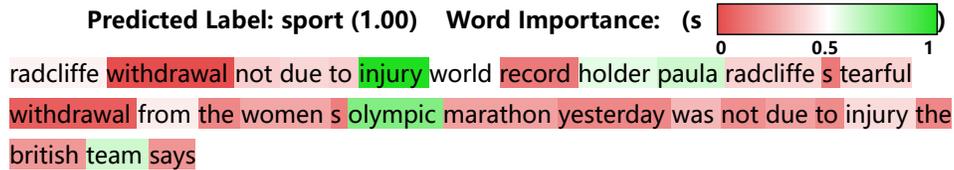

Figure 4: Attribution by **DeepFaith** for an LSTM predicting *sport* news on the AGNews dataset.

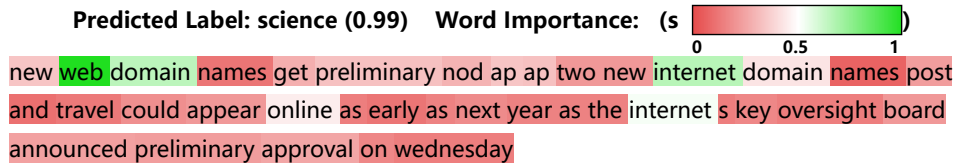

Figure 5: Attribution by **DeepFaith** for an LSTM predicting *science* news on the AGNews dataset.

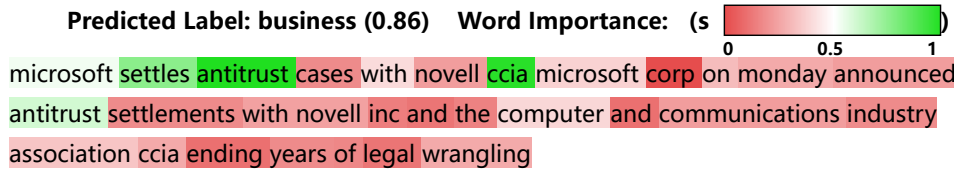

Figure 6: Attribution by **DeepFaith** for an LSTM predicting *business* news on the AGNews dataset.

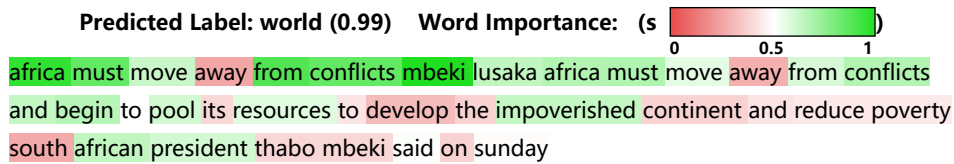

Figure 7: Attribution by **DeepFaith** for an LSTM predicting *world* news on the AGNews dataset.

9

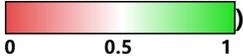

Figure 8: Attribution by **DeepFaith** for a Transformer predicting a review as the *positive* category on the IMDb dataset.

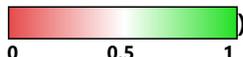

Figure 9: Attribution by **DeepFaith** for a Transformer predicting a review as the *negative* category on the IMDb dataset.

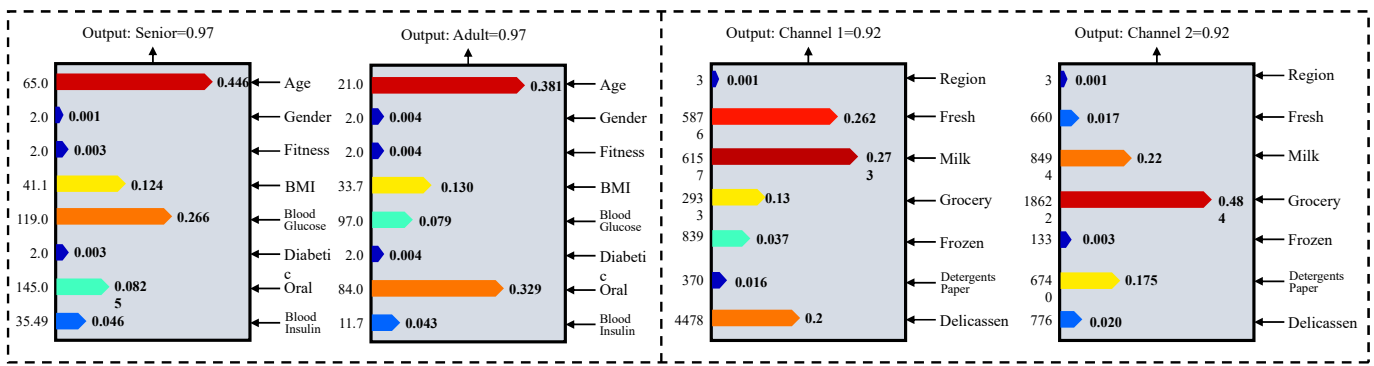

Figure 10: **DeepFaith** attributing four feature vectors predicted as different classes on the NAP (left) and WCD (right) datasets, respectively.

Figures 2 and 3 show saliency maps generated by **DeepFaith** on image-modality datasets and comparisons to baselines, with **DeepFaith** focusing on more precise, concentrated semantic regions. Figures 4–9 show that its attributions carry clear classification-relevant semantics. Figure 10 further demonstrates model fairness: on the NAP dataset, age-group classification relies mainly on *age* rather than *gender*.

10

# I Runtime Comparisons

| Method | OCT | | | ImageNet | | |
|---|---|---|---|---|---|---|
| | DeiT | EfficientNet | ResNet | DeiT | EfficientNet | ResNet |
| DeepFaith | **0.609** | **2.406** | **2.103** | **3.103** | **1.207** | **2.403** |
| Integrated Grads | 78.425 | 48.764 | 103.721 | 95.132 | 52.537 | 110.638 |
| DeepLIFT | 3.618 | 14.649 | 15.003 | 14.918 | 14.417 | 15.428 |
| Saliency | 3.874 | 6.526 | 8.548 | 11.264 | 6.543 | 9.254 |
| Occlusion | 96.147 | 67.728 | 170.348 | 115.435 | 72.366 | 185.721 |
| LIME | 145.517 | 135.273 | 93.352 | 121.143 | 100.769 | 137.326 |
| Kernel SHAP | 69.891 | 65.862 | 63.114 | 68.946 | 71.627 | 69.642 |
| Input × Gradient | 3.747 | 2.584 | 3.134 | 3.839 | 1.964 | 3.378 |
| Guided Backprop | 3.438 | 2.561 | 3.372 | 4.849 | 1.978 | 3.673 |
| Grad-CAM | 4.139 | 9.163 | 9.617 | 13.756 | 8.739 | 10.539 |
| Score-CAM | 553.973 | 1037.456 | 3628.819 | 3609.633 | 1266.827 | 4203.115 |
| Grad-CAM++ | 4.389 | 3.957 | 3.769 | 6.048 | 3.652 | 4.548 |
| Expected Grads | 57.642 | 65.432 | 122.261 | 124.935 | 65.934 | 127.967 |
| DeepLIFT SHAP | 14.293 | 29.317 | 86.179 | 57.848 | 34.158 | 98.453 |
| LRP | 7.519 | 6.867 | 8.539 | 17.548 | 7.716 | 8.134 |

Table 23: Average runtime (in ms) of **DeepFaith** and baseline methods for explaining a single sample in the image modality.

| Method | IMDb | | AGNews | | NAP | WCD |
|---|---|---|---|---|---|---|
| | LSTM | Transformer | LSTM | Transformer | MLP | MLP |
| DeepFaith | 1.217 | 0.563 | **1.137** | **0.433** | **0.117** | 0.173 |
| Integrated Grads | 3.125 | 21.926 | 53.941 | 58.473 | 2.839 | 0.219 |
| Gradient SHAP | **0.781** | 1.302 | 2.938 | 3.983 | 0.331 | **0.135** |
| DeepLIFT | 2.446 | 0.539 | 3.101 | 0.849 | 0.272 | 0.193 |
| Saliency | 1.278 | **0.238** | 5.894 | 0.682 | 0.254 | 0.137 |
| Occlusion | 10.872 | 0.587 | 61.725 | 25.734 | 0.563 | 0.293 |
| Feature Ablation | 21.921 | 0.253 | 60.987 | 25.768 | 0.337 | 0.293 |
| LIME | 59.067 | 15.954 | 79.311 | 112.438 | 16.125 | 22.443 |
| Kernel SHAP | 152.933 | 20.539 | 79.645 | 106.965 | 37.575 | 49.441 |

Table 24: Average runtime (in ms) of **DeepFaith** and baseline methods for explaining a single sample in text and tabular modality.

**DeepFaith** differs from conventional post-hoc attribution methods by incurring a one-time cost for generating supervised explanation signals and training its explainer. Once trained, however, it delivers explanations with low latency, making it well-suited for time-critical applications.

Across image, text, and tabular modalities, **DeepFaith** consistently produces explanations faster than sampling-based methods such as LIME, Kernel SHAP, and Occlusion, which require repeated perturbations and evaluations. It also outperforms most gradient-based approaches, including Integrated Grads, Grad-CAM, and Grad-CAM++, by avoiding repeated backpropagation through the explained model. This advantage becomes more pronounced for complex architectures, where the runtime of many baselines scales with model size, while **DeepFaith** remains largely unaffected due to its decoupling from the architecture of the target model.

The results demonstrate that **DeepFaith**'s inference speed is not only competitive but often superior across different backbones and modalities. Moreover, its consistent efficiency ensures that the same explainer design can be deployed in varied domains without sacrificing latency, enabling real-time interpretability in settings where both speed and explanation quality are critical.

## J   Ablation Study

| Task | Ablation | FC ↑ | FE ↑ | MC ↑ | RP ↑ | INS ↑ | DEL ↓ | NEG ↑ | POS ↓ | IROF ↑ | INF ↑ |
|---|---|---|---|---|---|---|---|---|---|---|---|
| OCT+DeiT | $\mathcal{L}_{\text{OBJ}}$ | 0.217 | 0.475 | 0.897 | 0.643 | 0.944 | 0.356 | 0.917 | 0.368 | 0.638 | 0.089 |
|  | $\mathcal{L}_{\text{PC}}$ | 0.032 | 0.231 | 0.655 | 0.540 | 0.913 | 0.463 | 0.904 | 0.521 | 0.534 | 0.031 |
|  | $\mathcal{L}_{\text{LC}}$ | 0.101 | 0.104 | 0.240 | 0.169 | 0.763 | 0.830 | 0.809 | 0.813 | 0.162 | 0.023 |
| OCT+EfficientNet | $\mathcal{L}_{\text{OBJ}}$ | 0.060 | 0.784 | 0.959 | 0.759 | 0.572 | 0.240 | 0.339 | 0.161 | 0.747 | 0.029 |
|  | $\mathcal{L}_{\text{PC}}$ | 0.056 | 0.749 | 0.910 | 0.750 | 0.566 | 0.245 | 0.325 | 0.191 | 0.704 | -0.021 |
|  | $\mathcal{L}_{\text{LC}}$ | 0.020 | 0.098 | 0.266 | 0.706 | 0.379 | 0.278 | 0.216 | 0.235 | 0.708 | 0.006 |
| OCT+ResNet | $\mathcal{L}_{\text{OBJ}}$ | 0.135 | 0.534 | 0.942 | 0.744 | 0.863 | 0.248 | 0.655 | 0.242 | 0.742 | 0.015 |
|  | $\mathcal{L}_{\text{PC}}$ | 0.004 | 0.215 | 0.532 | 0.696 | 0.860 | 0.285 | 0.565 | 0.282 | 0.704 | -0.004 |
|  | $\mathcal{L}_{\text{LC}}$ | 0.009 | 0.201 | 0.219 | 0.505 | 0.783 | 0.485 | 0.575 | 0.505 | 0.506 | 0.009 |
| ImageNet+DeiT | $\mathcal{L}_{\text{OBJ}}$ | 0.026 | 0.447 | 0.884 | 0.486 | 0.568 | 0.127 | 0.417 | 0.295 | 0.672 | 0.014 |
|  | $\mathcal{L}_{\text{PC}}$ | 0.022 | 0.364 | 0.823 | 0.456 | 0.501 | 0.185 | 0.406 | 0.366 | 0.638 | 0.008 |
|  | $\mathcal{L}_{\text{LC}}$ | -0.047 | -0.051 | 0.033 | 0.373 | 0.552 | 0.380 | 0.397 | 0.414 | 0.493 | -0.037 |
| ImageNet+EfficientNet | $\mathcal{L}_{\text{OBJ}}$ | 0.021 | 0.217 | 0.835 | 0.591 | 0.525 | 0.174 | 0.405 | 0.156 | 0.749 | 0.004 |
|  | $\mathcal{L}_{\text{PC}}$ | -0.022 | 0.036 | 0.818 | 0.551 | 0.505 | 0.178 | 0.329 | 0.274 | 0.704 | -0.010 |
|  | $\mathcal{L}_{\text{LC}}$ | 0.002 | -0.016 | 0.141 | 0.458 | 0.490 | 0.272 | 0.326 | 0.346 | 0.620 | 0.003 |
| ImageNet+ResNet | $\mathcal{L}_{\text{OBJ}}$ | 0.031 | 0.254 | 0.938 | 0.677 | 0.577 | 0.106 | 0.471 | 0.182 | 0.871 | 0.019 |
|  | $\mathcal{L}_{\text{PC}}$ | 0.001 | 0.092 | 0.719 | 0.605 | 0.518 | 0.138 | 0.388 | 0.271 | 0.815 | 0.003 |
|  | $\mathcal{L}_{\text{LC}}$ | -0.002 | 0.035 | -0.072 | 0.547 | 0.404 | 0.195 | 0.303 | 0.312 | 0.732 | -0.019 |
| IMDb+LSTM | $\mathcal{L}_{\text{OBJ}}$ | 0.172 | 0.486 | 0.360 | 0.812 | 0.872 | 0.151 | 0.869 | 0.182 | 0.813 | 0.038 |
|  | $\mathcal{L}_{\text{PC}}$ | 0.079 | 0.325 | 0.149 | 0.492 | 0.810 | 0.459 | 0.898 | 0.337 | 0.497 | 0.036 |
|  | $\mathcal{L}_{\text{LC}}$ | 0.023 | 0.051 | 0.298 | 0.791 | 0.855 | 0.664 | 0.833 | 0.585 | 0.299 | 0.023 |
| IMDb+Transformer | $\mathcal{L}_{\text{OBJ}}$ | 0.162 | 0.495 | 0.203 | 0.759 | 0.806 | 0.189 | 0.799 | 0.205 | 0.742 | 0.047 |
|  | $\mathcal{L}_{\text{PC}}$ | 0.058 | 0.358 | 0.195 | 0.718 | 0.784 | 0.192 | 0.775 | 0.344 | 0.655 | 0.038 |
|  | $\mathcal{L}_{\text{LC}}$ | 0.023 | 0.235 | 0.167 | 0.316 | 0.667 | 0.708 | 0.738 | 0.652 | 0.223 | 0.013 |
| AGNews+LSTM | $\mathcal{L}_{\text{OBJ}}$ | 0.363 | 0.597 | 0.629 | 0.648 | 0.919 | 0.297 | 0.906 | 0.256 | 0.650 | 0.275 |
|  | $\mathcal{L}_{\text{PC}}$ | 0.231 | 0.519 | 0.597 | 0.618 | 0.905 | 0.273 | 0.821 | 0.297 | 0.615 | 0.135 |
|  | $\mathcal{L}_{\text{LC}}$ | 0.030 | 0.069 | 0.352 | 0.120 | 0.554 | 0.828 | 0.556 | 0.822 | 0.116 | 0.055 |
| AGNews+Transformer | $\mathcal{L}_{\text{OBJ}}$ | 0.111 | 0.318 | 0.464 | 0.663 | 0.901 | 0.277 | 0.864 | 0.194 | 0.651 | 0.082 |
|  | $\mathcal{L}_{\text{PC}}$ | 0.075 | 0.044 | 0.356 | 0.598 | 0.891 | 0.347 | 0.857 | 0.307 | 0.584 | 0.075 |
|  | $\mathcal{L}_{\text{LC}}$ | 0.012 | 0.048 | -0.401 | 0.117 | 0.575 | 0.835 | 0.474 | 0.828 | 0.091 | 0.048 |
| NGP+MLP | $\mathcal{L}_{\text{OBJ}}$ | 0.788 | 0.763 | 0.952 | 0.957 | 0.844 | 0.031 | 0.770 | 0.031 | 0.844 | 0.238 |
|  | $\mathcal{L}_{\text{PC}}$ | 0.674 | 0.671 | 0.558 | 0.424 | 0.358 | 0.514 | 0.227 | 0.541 | 0.361 | 0.025 |
|  | $\mathcal{L}_{\text{LC}}$ | 0.748 | 0.515 | 0.535 | 0.426 | 0.360 | 0.442 | 0.512 | 0.124 | 0.638 | 0.135 |
| WCD+MLP | $\mathcal{L}_{\text{OBJ}}$ | 0.961 | 0.961 | 0.679 | 0.364 | 0.723 | 0.551 | 0.575 | 0.455 | 0.306 | 0.929 |
|  | $\mathcal{L}_{\text{PC}}$ | 0.936 | 0.935 | 0.529 | 0.343 | 0.529 | 0.658 | 0.394 | 0.690 | 0.297 | 0.706 |
|  | $\mathcal{L}_{\text{LC}}$ | 0.952 | 0.953 | 0.469 | 0.335 | 0.672 | 0.579 | 0.524 | 0.463 | 0.281 | 0.893 |

Table 25: Ablation study of **DeepFaith** on explanation tasks across all 12 explanation tasks. The table reports the average scores over ten faithfulness evaluation metrics, where $\mathcal{L}_{\text{OBJ}}$ denotes the explainer trained with both loss terms.

Table 25 presents the ablation study of **DeepFaith** across 12 explanation tasks covering image, text, and tabular modalities. We evaluate three training configurations for the explainer: (1) $\mathcal{L}_{\text{OBJ}}$, the full objective with both loss terms $\mathcal{L}_{\text{PC}}$ and $\mathcal{L}_{\text{LC}}$; (2) $\mathcal{L}_{\text{PC}}$ only; and (3) $\mathcal{L}_{\text{LC}}$ only. The table reports average scores over ten faithfulness evaluation metrics.

Across nearly all settings, the full objective $\mathcal{L}_{\text{OBJ}}$ achieves the highest faithfulness, often with large margins over single-loss variants. Using $\mathcal{L}_{\text{PC}}$ alone generally produces moderately faithful explanations but struggles to surpass strong baseline methods, particularly on correlation-sensitive metrics such as FC and FE. This indicates that while $\mathcal{L}_{\text{PC}}$ stabilizes training and captures global patterns, it lacks the local fidelity constraints necessary for optimal performance. Conversely, $\mathcal{L}_{\text{LC}}$ alone frequently underperforms and in several cases fails to optimize effectively, as seen in low or even negative scores for metrics like MC and RP, confirming that local consistency loss without the stability provided by $\mathcal{L}_{\text{PC}}$ is insufficient for convergence.

The complementarity of the two loss terms is evident: $\mathcal{L}_{\text{PC}}$ provides a stable optimization landscape, while $\mathcal{L}_{\text{LC}}$ enforces fine-grained, instance-level alignment with model behavior. Their joint use enables **DeepFaith** to achieve consistently high faithfulness across all modalities, architectures, and metric types, validating our theoretical claim that both global and local constraints are critical for producing explanations that are both accurate and robust.



\bibliography{aaai2026}

\end{document}